\documentclass{article}
\usepackage[preprint]{neurips_2024}

\usepackage[utf8]{inputenc}
\usepackage[T1]{fontenc}
\usepackage{amsmath,amssymb,amsfonts,amsthm}
\usepackage{booktabs}
\usepackage{graphicx}
\usepackage{mathtools}
\usepackage{microtype}
\usepackage{natbib}
\usepackage{subcaption}
\usepackage{array}
\usepackage{url}
\usepackage{xcolor}
\usepackage{hyperref}
\usepackage{enumitem}
\usepackage{mdframed}
\usepackage{algorithm}
\usepackage{algpseudocode}

\newtheorem{theorem}{Theorem}[section]
\newtheorem{proposition}[theorem]{Proposition}
\newtheorem{lemma}[theorem]{Lemma}

\theoremstyle{definition}
\newtheorem{definition}{Definition}

\theoremstyle{remark}

\newmdenv[
  backgroundcolor=gray!5,
  linecolor=black,
  linewidth=0.6pt,
  roundcorner=2pt,
  innertopmargin=6pt,
  innerbottommargin=6pt,
  innerleftmargin=8pt,
  innerrightmargin=8pt,
  skipabove=8pt,
  skipbelow=8pt
]{anchorresult}

\newcommand{\E}{\mathbb E}

\newcommand{\R}{\mathbb R}
\newcommand{\D}{\mathcal D}

\newcommand{\G}{\mathcal G}

\newcommand{\U}{\mathcal U}
\newcommand{\M}{\mathcal M}

\newcommand{\Wass}{W_1}
\newcommand{\MDE}{\operatorname{MDE}}
\newcommand{\argmin}{\operatorname*{arg\,min}}

\title{Choosing Online Experiment Designs under Interference in Ads, Recommendations, and Member-Experience Systems}

\author{%
  Prashant Shekhar\thanks{Corresponding author.} and Caroline Howard \\
  \textit{\fontsize{10}{22}\selectfont Department of Mathematics} \\
  \textit{\fontsize{10}{22}\selectfont Embry-Riddle Aeronautical University} \\
  \textit{\fontsize{10}{22}\selectfont Daytona Beach, FL, USA}
}

\begin{document}
\maketitle

\begin{abstract}
Online experiments in ads, recommendation, and member-experience systems are often planned before the dominant interference mechanism is known. A treatment may propagate through budgets, inventory, producer exposure, graph spillovers, or temporal carryover, making the randomization design itself a statistical decision. We formulate this problem as robust design selection over uncertain exposure mechanisms. Given a finite catalog of six implementable designs, the selector compares each design by worst-case planning risk over an ambiguity set. The risk combines exposure bias, assignment-unit variance, minimum detectable effect, contamination or carryover, operational cost, and estimand mismatch. For theoretical justification, the paper develops a geometry-aware guarantee, stating that design bias is bounded by Wasserstein distance to the launch exposure distribution, and this penalty is minimax tight under Lipschitz exposure response. We also prove finite-catalog approximation and a robust selector theorem with excess-risk control, exact recovery under separation, and certified shortlists when the risk surface is flat. Empirically, the same selector gives different recommendations across samples from public datasets. It selects \textit{user randomization} on Criteo ads with dimensionless robust risk \(1.295\), \textit{switchbacks} on Open Bandit \texttt{bts/men} with risk \(2.105\), and \textit{cluster randomization} on KuaiRand with risk \(2.240\). The Open Bandit case stresses known but uneven logging support, with propensities from \(0.00006\) to \(0.594\) and a \(5.17\%\) IPS effective-sample share. A controlled regime sweep further shows phase transitions from user randomization to mixed randomization, cluster randomization, and switchbacks as exposure mechanisms change. Overall, the paper contributes an interference-aware experiment design framework based on mechanism-robust design decisions, where the output is either a justified design choice or an uncertainty shortlist.
\end{abstract}

\section{Introduction}

Online controlled experiments are central to product and machine-learning decisions in large platforms \citep{kohavi2013online,imbens2015causal}. Their design is least ambiguous when treatment affects only the assigned unit. However, modern ads, recommendation, and member-experience systems rarely have that structure. A ranking change can alter producer exposure, an ad-policy change can move advertiser budgets across inventory, and a member-experience treatment can affect future platform state. Before an experiment is launched, the platform may not know whether the relevant exposure mechanism is graph spillover, shared marketplace state, budget pacing, temporal carryover, or a combination of these channels. To model this, the current research formulates online experiment design under interference as a robust design-selection problem over uncertain exposure mechanisms. The input is a finite catalog of implementable designs, such as user randomization, graph or market clustering, switchbacks, budget splits, saturation designs, and mixed randomization. The ambiguity set contains plausible exposure mechanisms. Each candidate design is evaluated by the exposure distribution it induces under each mechanism and by the operational quantities needed to run it. The output is not a generic recommendation to ``use switchbacks'' or ``cluster users.'' Rather, it is a robust choice, or an uncertainty shortlist, for the specific ambiguity set and design catalog. This perspective is complementary to offline marketplace decision-support work, which argues that logged replay and off-policy evidence should often culminate in online validation rather than direct launch when propensities, bidder response, and interference remain unresolved \citep{shekhar2026decision}. The present paper studies the downstream design question that follows from that conclusion. Once validation is required, the platform still has to choose a randomization design before knowing which exposure mechanism will dominate.

This focus changes the role of interference theory. Prior work often asks how to analyze a chosen design. Here the primary question is which design should be chosen before the dominant mechanism is known. A design with low variance can be poor if it induces exposure far from launch. A design that better matches launch exposure can be poor if it has too few assignment units or excessive carryover. The framework makes these tradeoffs explicit through a planning risk that combines exposure bias, assignment-unit variance, minimum detectable effect, contamination or carryover risk, operational cost, and estimand mismatch. Figure~\ref{fig:regime-transition} presents this central phenomenon through a controlled (simulated) mechanism-sweep diagnostic using the proposed design catalog and normalized risk components. The sweep varies a one-dimensional exposure-mechanism intensity \(\gamma\), moving from weak row-local interference to mixed spillover, clustered spillover, and carryover-dominant regimes. Holding the catalog and risk weights fixed, the lowest-risk design changes as the exposure mechanism changes. User randomization is attractive when interference is weak, mixed randomization wins in a moderate spillover regime, cluster randomization wins when shared exposure dominates, and switchbacks become preferable when carryover is the main threat. The figure grounds the central research agenda that design choice should be made as a robust selection problem over uncertain exposure mechanisms.

\begin{figure}[t]
\centering
\includegraphics[width=0.88\linewidth]{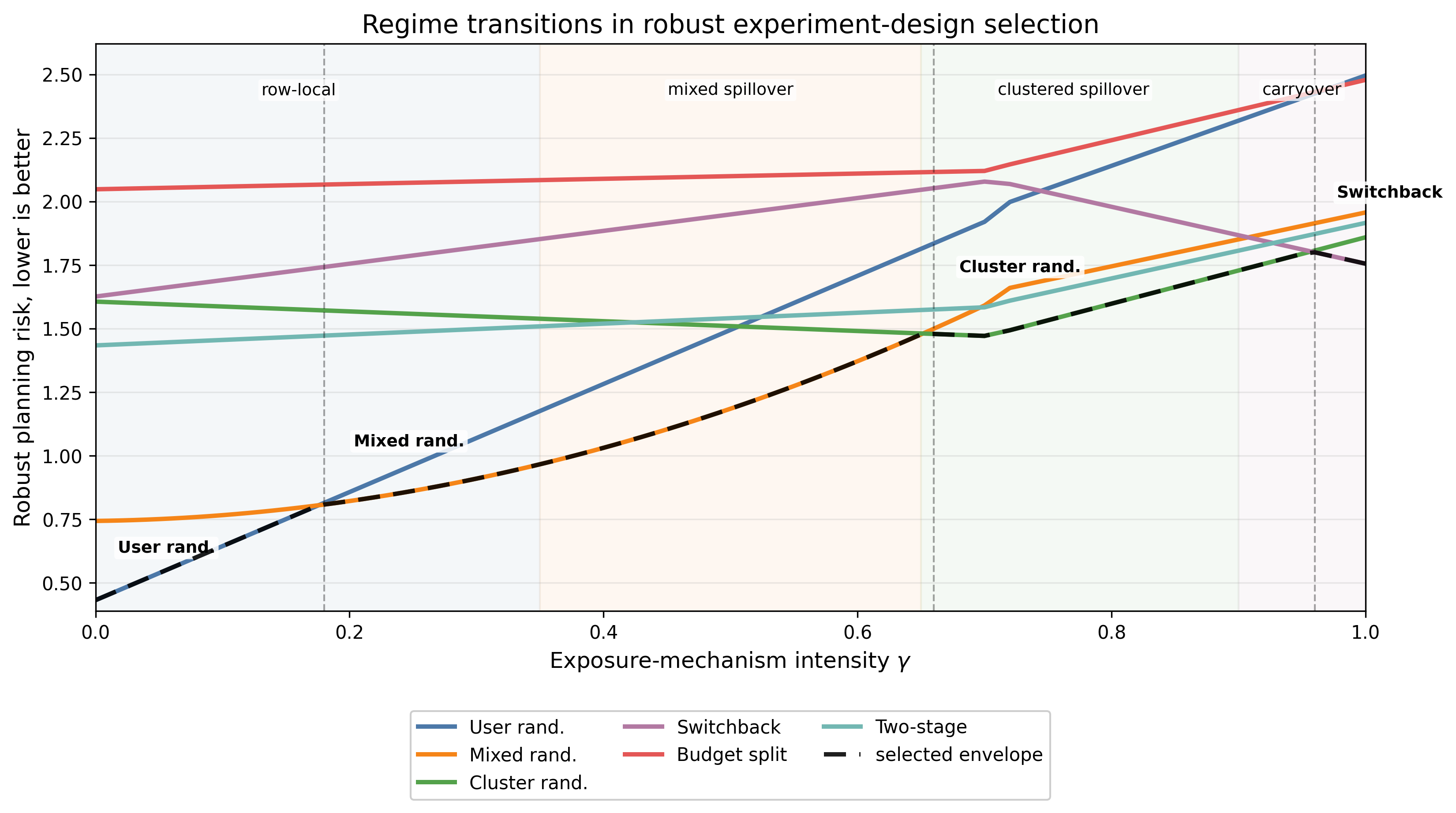}
\caption{\small Regime-transition generated from the proposed experiment-design framework. The diagnostic uses the same six implementable designs (shown at the bottom), normalized risk components, and risk weights used in the empirical selector, while sweeping a controlled exposure-mechanism intensity \(\gamma\) from weak row-local interference to mixed spillover, clustered spillover, and carryover-dominant interference. The dashed black envelope marks the lowest-risk design at each value of \(\gamma\). The crossings show why no single randomization scheme is uniformly preferred.}
\label{fig:regime-transition}
\end{figure}

The paper makes four contributions in service of this formulation. \textbf{First}, it defines a deployable robust design selector whose inputs are a design catalog, an exposure ambiguity set, historical logs, planning weights, and a shortlist tolerance. \textbf{Second}, it introduces exposure geometry as the common scale for comparing designs that randomize different objects. Under Lipschitz response uncertainty, the Wasserstein distance between design-induced and launch exposure distributions controls design bias, and the bound is worst-case tight. \textbf{Third}, it justifies finite design catalogs through an approximation proposition for catalog coverage in design space. \textbf{Fourth}, it proves that the selector in Algorithm~\ref{alg:robust-selector} has excess-risk control, exact recovery when robust risks are separated, and certified shortlists when the risk surface is too flat for a unique recommendation. The empirical analysis uses samples from public ads and recommendation datasets with semi-synthetic interference layers to validate the same algorithmic object.

\section{Related Work}

\noindent \textbf{Experimentation infrastructure}:
Large-scale experimentation systems provide the operational base for this research. Industry practice covers metric definition, guardrails, assignment, and platform execution \citep{kohavi2013online}. CUPED-style variance reduction and sequential monitoring improves sensitivity and governance after a design has been chosen \citep{deng2013cuped,kharitonov2015sequential}. Assignment systems add implementation constraints such as deterministic hashing, independence across concurrent experiments, and stable treatment assignment \citep{li2022assignvariants}. These works motivate the cost, feasibility, and minimum detectable effect terms in our planning risk, but they usually take the assignment unit as fixed.

\noindent \textbf{Interference-aware causal designs}:
Potential-outcomes work under interference replaces the no-interference assumption with exposure mappings, partial interference, or general interference structures \citep{hudgens2008interference,tchetgen2012causal,aronow2017general}. Design-based limits under arbitrary interference explain why ordinary A/B reasoning can be fragile when exposure is not localized or measured \citep{basse2018limitations}. Graph cluster randomization, two-stage designs, saturation designs, mixed Bernoulli-cluster designs, unknown-network methods, and diffusion randomization each provide valid tools for particular interference structures \citep{ugander2013graph,eckles2017design,ugander2020randomized,basse2018two,viviano2025causalclustering,jiang2024mixture,yu2022unknownnetwork,fatemi2024cascade}. We build on this literature by treating the exposure mechanism itself as uncertain and by comparing designs through the exposure distributions they induce.

\noindent \textbf{Temporal, marketplace, and recommender mechanisms}:
Switchback experiments are useful when treatments perturb shared state over time, and formal analyses account for temporal dependence, carryover, and multiple randomization units \citep{bojinov2023switchback,masoero2023efficient,missault2025switchback,bajari2021multiple}. Two-sided platforms add market-balance, supply-competition, budget, and pacing channels, motivating demand-side, supply-side, budget-split, and multiple-randomization designs \citep{johari2022experimental,li2022interference,liu2020budget,holtz2020limiting,holtz2020pricing}. Recent first-price pacing-equilibrium work formalizes why ads experiments can trade off bias, variance, and market behavior \citep{liao2023fppe,leonenkov2025fppe}. Recommendation systems create related producer-side and seller-side exposure channels even under independent viewer assignment \citep{nandy2021recommender,wang2025producer,zhu2024seller}. These domain literatures supply candidate mechanisms for our ambiguity set rather than a single universally correct design.

\noindent \textbf{Reproducible validation and gap}:
Public ads and recommendation datasets support reproducible stress tests for design-selection methods. We use Criteo, Open Bandit, KuaiRand, and MovieLens in the empirical analysis \citep{diemert2018criteo,saito2021open,gao2022kuairand,chen2022kuairand,harper2015movielens}; related public corpora such as MIND provide additional recommendation substrates for future validation \citep{wu2020mind}. The gap is that existing work gives many sound designs and estimators for specified mechanisms, but rarely makes the choice of design a robust decision problem over uncertain mechanisms. This paper fills that gap by representing each candidate design through its exposure distribution, comparing it to launch exposure by Wasserstein geometry, and selecting designs by worst-case planning risk.

\section{Proposed framework}

Let \(i=1,\ldots,n\) index measurement units such as users, members, sessions, advertisers, producers, or items. Let \(Z_i\in\{0,1\}\) denote the treatment assignment for unit \(i\), and let \(Z=(Z_1,\ldots,Z_n)\) denote the full assignment vector. Under no interference, unit \(i\)'s potential outcome is \(Y_i(Z_i)\). Under interference, the potential outcome can depend on the full assignment vector, written as \(Y_i(Z)\). This dependence is too rich to estimate directly, so experiment design usually imposes an exposure mapping.

\begin{definition}[Exposure mapping]
An exposure mapping is a function \(e_i(Z; \G, H_t)\) that summarizes the part of the assignment vector relevant for unit \(i\). The arguments may include a network or marketplace graph \(\G\), historical platform state \(H_t\), and time \(t\). A potential outcome is exposure-compatible if
\[
Y_i(Z) = Y_i\big(e_i(Z; \G, H_t)\big).
\]
\end{definition}

The exposure mapping can encode many practical mechanisms. In a social or producer-consumer graph, exposure may be the treated share of neighbors. In ads, it may be the treated share of advertiser budget competing for the same inventory. In a recommendation system, it may be the treated share of recommendation opportunities that expose a producer, item, or genre. In a switchback experiment, it may include current treatment and recent treatment history to capture carryover.

For a mechanism \(\theta\), let \(\tau^\star(\theta)\) denote the target launch estimand, usually the effect of applying the treatment platform-wide. Let \(d\in\D\) denote an experimental design, including an assignment rule, blocking scheme, analysis unit, and estimator. A design returns an estimator \(\widehat{\tau}_d\). The design-choice problem is to select \(d\) before the experiment, using domain knowledge and historical data, so that \(\widehat{\tau}_d\) is informative for \(\tau^\star(\theta)\) across the plausible mechanisms in \(\U\).

The key difficulty is that the exposure mapping is not known with certainty before launch. We index plausible mechanisms by \(\theta\in\U\), where \(\U\) is an ambiguity set constructed from product knowledge, historical logs, and stress-test simulations. A mechanism \(\theta\) specifies the relevant exposure map, the strength of spillover or carryover, and the launch exposure \(e_i^\star(\theta)\). A design \(d\) induces an exposure distribution \(e_i^d(Z;\theta)\) under its randomization rule. The target launch estimand is \(\tau^\star(\theta)\), while the expected estimand recovered by design \(d\) is denoted \(\tau_d(\theta)=\E_\theta[\widehat\tau_d]\). The resulting design bias is
\[
B_d(\theta)=\tau_d(\theta)-\tau^\star(\theta).
\]

The ambiguity set can be built in several ways, including a hand-specified stress grid, a confidence region around estimated spillover parameters, or a Bayesian posterior credible set. The implementation in this paper uses the first option because the goal is pre-launch design choice rather than posterior inference. A concrete construction is
\[
\U
=
\left\{
\theta=(\gamma_g,\gamma_b,\lambda,\ell):
\gamma_g\in\Gamma_g,\ 
\gamma_b\in\Gamma_b,\ 
\lambda\in\Lambda,\ 
\ell\in\mathcal L
\right\},
\]
where \(\gamma_g\) is graph or producer-side spillover strength, \(\gamma_b\) is budget or shared-market spillover strength, \(\lambda\) is temporal carryover strength, and \(\ell\) indexes the locality scale used in the exposure map. For example, an ads marketplace audit could use
\[
\begin{aligned}
\Gamma_g&=\{0,0.1,0.3\},&
\Gamma_b&=\{0,0.2,0.5\},\\
\Lambda&=\{0,0.05,0.2\},&
\mathcal L&=\{\text{campaign},\text{advertiser},\text{exchange-region}\}.
\end{aligned}
\]
The implementation treats these values as a pre-specified audit grid informed by product knowledge and coarse log summaries such as budget concentration, repeated exposure, cluster boundary mass, and daypart persistence. The logs do not estimate the grid points directly. Instead, they calibrate the case-specific outcome scale, assignment-unit structure, graph or budget shares, support stress, and repeated-interaction structure used to evaluate each grid point. Product knowledge can remove implausible combinations in deployment; in the implementation here, we retain the full \(3\times3\times3\times3\) audit grid. The robust selector then evaluates each design on this finite grid and optimizes the worst-case risk. The same framework would allow a statistically estimated \(\U\), but the paper deliberately uses an auditable grid so the design recommendation can be traced to named exposure mechanisms.

\begin{definition}[Design risk]
For design \(d\) and mechanism \(\theta\), define planning risk as
\begin{equation}
\label{eq:design-risk}
\mathcal R(d;\theta)
=
w_g \widetilde B_d(\theta)
+ w_v \widetilde V_d(\theta)
+ w_m \widetilde M_d(T;\theta)
+ w_c \widetilde C_d(\theta)
+ w_o \widetilde O_d
+ w_e \widetilde E_d(\theta),
\end{equation}
where \(B_d(\theta)\) is design bias, \(V_d(\theta)\) is the assignment-unit variance proxy used for power planning, \(M_d(T;\theta)\) is a planning-scale minimum detectable effect at duration \(T\), \(C_d(\theta)\) is contamination or carryover risk, \(O_d\) is operational cost, \(E_d(\theta)\) is residual estimand mismatch, tildes denote within-case normalized component scores, and all weights are nonnegative application-specific priorities.
\end{definition}

The risk in Eq.~\eqref{eq:design-risk} is not meant to reduce all design choice to one scalar. Rather, its purpose is to force the design conversation to name the quantities being traded off. The operational-cost term is a pre-registered score on \([0,1]\), normalized before weighting. A concrete implementation can set
\[
O_d
=
\frac{
  a_h\widetilde H_d+a_s\widetilde S_d+a_r\widetilde R_d+a_p\widetilde P_d
}{
  a_h+a_s+a_r+a_p
},
\]
where \(\widetilde H_d\) is normalized engineering effort, \(\widetilde S_d\) is orchestration complexity, \(\widetilde R_d\) is rollback or failure-mode risk, and \(\widetilde P_d\) is platform-integration burden. The weights \(a_h,a_s,a_r,a_p\) are fixed before selection. In the experiments in this research, we assign \(O_d=0.10\) to standard user-level randomization, \(O_d=0.35\) to \(0.45\) to cluster or switchback designs that require blocking, scheduling, or washout logic, and \(O_d=0.70\) to \(0.90\) to budget splits, saturation designs, or mixed randomization that require new allocation rules or multi-axis orchestration. A production implementation could replace the ordinal inputs with engineering-hour estimates and formal launch-risk reviews. In a high-stakes launch, bias and contamination may dominate. In a low-risk UI test, variance and implementation cost may be more important. In ads and recommendation systems, the most important design is often not the one with the smallest standard error, but the one whose estimand is closest to the launch decision.

\begin{definition}[Robust design objective]
Given a feasible design set \(\D_{\mathrm{feas}}\) and ambiguity set \(\U\), the robust design target is
\begin{equation}
\label{eq:robust-objective}
d^\star_{\mathrm{rob}}
\in
\argmin_{d\in\D_{\mathrm{feas}}}
\sup_{\theta\in\U}\mathcal R(d;\theta).
\end{equation}
\end{definition}

The robust objective in Eq.~\eqref{eq:robust-objective} is the paper's organizing optimization problem. If \(\U\) contains only one known mechanism, the problem reduces to ordinary design optimization for that mechanism. If \(\U\) contains several plausible mechanisms, the selected design must be credible even when the dominant interference channel is not known exactly.

\noindent \textbf{Operational ambiguity-set example}:
Consider an ads ranking or pacing change whose launch effect may propagate through advertiser budgets, producer-side exposure, and short-run carryover. A pre-launch design review can instantiate the ambiguity set as follows. First, define the exposure feature
\[
\phi_i^\theta(Z)
=
\big(
Z_i,\ 
\gamma_b\,\bar Z_{\mathrm{budget}(i),t},\
\gamma_g\,\bar Z_{\mathrm{nbr}(i),t},\
\lambda\, Z_{\mathrm{surface}(i),t-1}
\big),
\]
where \(\bar Z_{\mathrm{budget}(i),t}\) is the treated share of auctions drawing on the same budget or pacing pool, \(\bar Z_{\mathrm{nbr}(i),t}\) is the treated share in the relevant producer, advertiser, or item neighborhood, and \(Z_{\mathrm{surface}(i),t-1}\) records lagged treatment on the same market surface. Second, set the stress grid
\[
\gamma_b\in\{0,0.2,0.5\},\qquad
\gamma_g\in\{0,0.1,0.3\},\qquad
\lambda\in\{0,0.05,0.2\}.
\]
These values mean no effect, moderate effect, and severe but credible effect for each channel. Third, for each design \(d\) in the catalog, simulate or replay its assignment rule on historical logs and compute
\[
\widehat G_d(\theta),\quad
\widehat V_d(\theta),\quad
\widehat M_d(T;\theta),\quad
\widehat C_d(\theta),\quad
\widehat O_d,\quad
\widehat E_d(\theta).
\]
Here \(G_d(\theta)\) measures the exposure-geometry distance between the experiment's exposure and the launch exposure distributions. It is the observable planning proxy for design bias \(B_d(\theta)\), which is not observed before launch. The implementation uses a paired exposure-distance proxy in the normalized exposure-feature coordinates, while the theory below treats \(G_d(\theta)\) as the corresponding Wasserstein geometry object. In the experiments, the normalized planning weights are
\[
(w_g,w_v,w_m,w_c,w_o,w_e)
=
(1.00,0.80,0.75,0.45,0.45,0.65).
\]
Let \(\widehat{\mathcal R}_T(d;\theta)\) denote the estimated planning risk obtained by aggregating these normalized components with the weights above; Section~\ref{sec:robust-selection} specializes this object to the geometry-aware risk used in the formal selector. The selector then computes the estimated risk \(\widehat{\mathcal R}_T(d;\theta)\) for every design-mechanism pair and reports
\[
\widehat d_T
\in
\argmin_{d\in\D_{\mathrm{cat}}}
\max_{\theta\in\U}
\widehat{\mathcal R}_T(d;\theta),
\]
together with the uncertainty shortlist. This is the proposed pipeline in the current framework, with the public datasets providing the case-specific ingredients available in each setting, such as outcomes, triggered exposure, logged propensities, assignment-unit structure, and repeated interactions.

\begin{algorithm}[t]
\caption{Robust Design Selector}
\label{alg:robust-selector}
\small
\begin{algorithmic}[1]
\Require design catalog \(\D_{\mathrm{cat}}\), ambiguity set \(\U\), logs \(\mathcal L\), horizon \(T\), weights \(w=(w_g,w_v,w_m,w_c,w_o,w_e)\)
\Ensure selected design \(\widehat d_T\), shortlist \(\widehat{\mathcal S}_{2\epsilon}\), planning tolerance \(2\epsilon_T\)

\State Specify exposure features \(\phi_i^\theta(Z)\) for each \(\theta\in\U\), including direct treatment, shared-resource exposure, graph or item-side exposure, and lagged treatment when relevant.

\ForAll{\(d\in\D_{\mathrm{cat}}\)}
  \ForAll{\(\theta\in\U\)}
    \State Simulate or replay the assignment rule for \(d\) on logs \(\mathcal L\) under mechanism \(\theta\).
    \State Estimate components
    \[
    \widehat G_d(\theta),\ 
    \widehat V_d(\theta),\ 
    \widehat M_d(T;\theta),\ 
    \widehat C_d(\theta),\ 
    \widehat O_d,\ 
    \widehat E_d(\theta).
    \]
  \EndFor
\EndFor

\State Normalize each component within the empirical case before aggregation.
\State Let \(\widetilde G_d,\widetilde V_d,\widetilde M_d,\widetilde C_d,
\widetilde O_d,\widetilde E_d\) denote the resulting unit-free scores.

\ForAll{\(d\in\D_{\mathrm{cat}},\theta\in\U\)}
  \State Compute
  \[
  \widehat{\mathcal R}_T(d;\theta)
  =
  w_g\widetilde G_d(\theta)
  +w_v\widetilde V_d(\theta)
  +w_m\widetilde M_d(T;\theta)
  +w_c\widetilde C_d(\theta)
  +w_o\widetilde O_d
  +w_e\widetilde E_d(\theta).
  \]
\EndFor

\ForAll{\(d\in\D_{\mathrm{cat}}\)}
  \State \(\widehat Q_T(d)\gets \max_{\theta\in\U}\widehat{\mathcal R}_T(d;\theta)\).
\EndFor

\State \(\widehat d_T\gets\argmin_{d\in\D_{\mathrm{cat}}}\widehat Q_T(d)\).
\State Estimate planning-error budget \(\epsilon_T\) from component uncertainty, sensitivity checks, or a pre-specified conservative fraction of \(\widehat Q_T(\widehat d_T)\). Here we use $
\epsilon_T
=
0.10\,\widehat Q_T(\widehat d_T).
$
\State \(\widehat{\mathcal S}_{2\epsilon}
\gets
\{d\in\D_{\mathrm{cat}}:\widehat Q_T(d)\le \widehat Q_T(\widehat d_T)+2\epsilon_T\}\).
\State \Return \(\widehat d_T,\widehat{\mathcal S}_{2\epsilon},2\epsilon_T\).
\end{algorithmic}
\end{algorithm}

Algorithm~\ref{alg:robust-selector} makes the deployable framework explicit. The ambiguity set names the mechanisms the platform is unwilling to rule out, the logs populate the risk components, and the theorem later in the paper explains when the returned design or shortlist has a finite-sample robust-risk interpretation.

\noindent\textbf{Candidate Design Catalog}: The framework assumes that the platform can enumerate a finite catalog of feasible designs before launch. Table~\ref{tab:design-menu} is the catalog used in the paper. Here, each row is a candidate decision in the robust selector, and each row induces a different exposure distribution, assignment-unit variance, and operational cost under a plausible mechanism.

\begin{table}[t]
\centering
\small
\caption{Candidate designs used by the robust selector. Each design is treated as an implementable action with its own exposure geometry and planning-risk components.}
\label{tab:design-menu}
\setlength{\tabcolsep}{3pt}
\renewcommand{\arraystretch}{1.08}
\begin{tabular}{>{\raggedright\arraybackslash}p{0.17\linewidth}>{\raggedright\arraybackslash}p{0.23\linewidth}>{\raggedright\arraybackslash}p{0.25\linewidth}>{\raggedright\arraybackslash}p{0.22\linewidth}}
\toprule
Design & Assignment unit & Best suited for & Main risk \\
\midrule
User randomization & User or member & Low-interference product changes with independent outcomes & Contamination through shared supply, recommendations, budgets, or social spillovers \\
Cluster randomization & Graph, region, producer, advertiser, or market cluster & Network or marketplace spillovers with observable clusters & Lower power and imperfect cluster separation \\
Switchback & Time block within region, market, or product surface & Shared marketplace state, inventory, and ranking systems with fast reset & Carryover, temporal shocks, and daypart imbalance \\
Budget split & Advertiser, campaign, or budget universe & Ads systems with budget and pacing interference & Engineering complexity and estimand change \\
Two-stage saturation & Cluster first, saturation level second & Estimating direct and spillover effects separately & Operational complexity and larger sample requirements \\
Mixed randomization & Multiple assignment axes & Systems with multiple interacting unit types & Design and analysis complexity \\
\bottomrule
\end{tabular}
\end{table}

The catalog view is deliberately restrictive. A production platform rarely optimizes over all randomization schemes. It selects from designs that can be implemented, monitored, and explained to stakeholders. The theoretical question is therefore whether this finite catalog is rich enough to approximate a broader design family. Proposition~\ref{prop:catalog-approximation} gives the corresponding approximation guarantee, and the empirical program checks whether the catalog separates designs meaningfully under the observed exposure geometry and assignment-unit variance.

\noindent \textbf{Estimation details}: All components in Algorithm~\ref{alg:robust-selector} are computed from the same replayed assignment table. For each design \(d\) and mechanism \(\theta\), the framework simulates the assignment rule on historical logs and constructs exposure features \(\phi_i^d(Z;\theta)\). The exposure-distance component is computed as a distance between the design-induced exposure distribution and the target launch exposure distribution,
\[
\widehat G_d(\theta)
=
D\!\left(\widehat P_d^\theta,\widehat P_\star^\theta\right),
\]
where \(D\) is the exposure-distance metric used in the empirical case. In the framework, \(D\) is the paired \(L_1\) gap between experimental exposure and full-launch exposure over normalized direct-treatment, budget-exposure, cluster-exposure, and lagged-treatment coordinates. This is the empirical proxy used to populate the Wasserstein geometry term. The assignment-unit variance is computed after aggregating outcomes to the design's assignment unit \(a\in\mathcal A_d\),
\[
\widehat V_d(\theta)
=
\frac{1}{|\mathcal A_d|-1}
\sum_{a\in\mathcal A_d}
\left(Y_a-\bar Y_{\mathcal A_d}\right)^2.
\]
The planning MDE for horizon \(T\) is
\[
\widehat M_d(T;\theta)
=
(z_{1-\alpha/2}+z_{1-\beta})
\sqrt{\frac{2\widehat V_d(\theta)}{N_d(T)}},
\]
where \(N_d(T)\) is the effective number of assignment units available over duration \(T\). Here \(z_{1-\alpha/2}\) is the standard-normal critical value for a two-sided test at significance level \(\alpha\), and \(z_{1-\beta}\) is the standard-normal critical value corresponding to power \(1-\beta\). The contamination or carryover score \(\widehat C_d(\theta)\) is computed from control-arm spillover exposure and lagged treatment switching under mechanism \(\theta\). The operational-cost score \(\widehat O_d\) is pre-specified from engineering complexity, orchestration burden, rollback risk, and platform integration. The estimand-mismatch score \(\widehat E_d(\theta)\) measures the average distance between the exposure profile induced by design \(d\) and the full-launch exposure profile under the same mechanism. In the Open Bandit adaptive-logging case, the known-propensity effective-sample-size loss is added as a support-stress adjustment to \(\widehat C_d(\theta)\) and \(\widehat E_d(\theta)\).

Before aggregation, each component is normalized within the empirical case. For a generic component \(X_d(\theta)\),
\[
\widetilde X_d(\theta)
=
\frac{\widehat X_d(\theta)}
{\max_{d'\in\D_{\mathrm{cat}},\,\theta'\in\U}
|\widehat X_{d'}(\theta')|}.
\]
The selector then forms
\[
\widehat Q_T(d)
=
\max_{\theta\in\U}
\left[
w_g\widetilde G_d(\theta)
+w_v\widetilde V_d(\theta)
+w_m\widetilde M_d(T;\theta)
+w_c\widetilde C_d(\theta)
+w_o\widetilde O_d
+w_e\widetilde E_d(\theta)
\right].
\]
Finally, the planning-error budget is set to
$
\epsilon_T
=
0.10\,\widehat Q_T(\widehat d_T),
$
so the reported shortlist is
\[
\widehat{\mathcal S}
=
\left\{
d\in\D_{\mathrm{cat}}:
\widehat Q_T(d)
\le
\widehat Q_T(\widehat d_T)+2\epsilon_T
\right\}.
\]

\section{Geometry-Aware Robust Design Selection}
\label{sec:robust-selection}

The framework begins with a product-mechanism audit. For a candidate treatment, the analyst identifies the unit being changed, the unit being measured, the shared resource that could create interference, and the time scale over which effects may carry over. This audit maps the experiment to a set of plausible mechanisms rather than to one assumed truth. The relevant mechanisms in this paper include local graph spillovers, shared marketplace state, budget and pacing spillovers, and temporal carryover.

\subsection{Exposure geometry}

The first ingredient is a way to compare designs that perturb different objects. A user-randomized design, a graph-cluster design, and a switchback design can all assign treatment, but they induce different exposure distributions. Under interference, this exposure distribution is the relevant object.

For mechanism
$
\theta=(\gamma_g,\gamma_b,\lambda,\ell),
$
let \(\phi_i^d(Z;\theta)\in(\M,\rho)\) denote the exposure feature for unit \(i\) under design \(d\), where \((\M,\rho)\) is a metric exposure-feature space. The feature may include direct treatment, graph or producer-side exposure, shared-budget exposure, and lagged treatment. The parameters \(\gamma_g\), \(\gamma_b\), and \(\lambda\) weight graph spillover, budget spillover, and temporal carryover. The locality index \(\ell\) controls the granularity at which exposure neighborhoods are constructed. Let \(\mathcal L_d(\phi_i^d(Z;\theta))\) be the law of the exposure feature induced by the random assignment rule of design \(d\). We define
$
P_d^\theta
=
\frac{1}{n}\sum_{i=1}^n
\mathcal L_d\!\left(\phi_i^d(Z;\theta)\right),
$ $
P_\star^\theta
=
\frac{1}{n}\sum_{i=1}^n
\delta_{\phi_i^\star(\theta)}
$
as the experimental exposure distribution and the full-launch exposure distribution. The population geometry term is
\begin{equation}
\label{eq:exposure-geometry}
G_d(\theta)
=
W_1(P_d^\theta,P_\star^\theta),
\end{equation}
the Wasserstein-1 distance between the exposure distribution created by the experiment and the exposure distribution expected under launch.

The empirical implementation of the geometry term uses the finite panel generated by each replayed design. For unit \(i\), the design produces direct treatment \(Z_i\), shared-budget exposure \(\bar Z_{\mathrm{budget}(i),t}\), graph or cluster exposure \(\bar Z_{\mathrm{graph}(i),t}\), and lagged treatment \(Z_{i,t-1}\). The full-launch benchmark is the all-treated exposure profile, so each of these exposure coordinates equals one under launch. The empirical geometry score used in the code is
\begin{equation}
\label{eq:empirical-exposure-geometry}
\widehat G_d(\theta)
=
\frac{1}{1+\gamma_g+\gamma_b+\lambda}
\cdot
\frac{1}{n}
\sum_{i=1}^n
\left[
|1-Z_i|
+
\gamma_b\left|1-\bar Z_{\mathrm{budget}(i),t}\right|
+
\gamma_g\left|1-\bar Z_{\mathrm{graph}(i),t}\right|
+
\lambda\left|1-Z_{i,t-1}\right|
\right].
\end{equation}
The denominator puts grid points with different spillover and carryover weights on a common scale. Thus \(\widehat G_d(\theta)\) is a computable proxy for Eq.~\eqref{eq:exposure-geometry}. It is small when the experimental design creates exposure profiles close to full launch, and large when the design leaves many units far from their launch exposure.

The theoretical analysis uses \(W_1(P_d^\theta,P_\star^\theta)\) because it yields a transport-based bias bound for Lipschitz exposure-response functions. The empirical implementation uses Eq.~\eqref{eq:empirical-exposure-geometry} because the replayed panel provides explicit exposure coordinates. Both objects express the same design principle that under interference, a design is safer when its experimental exposure distribution is close to the exposure distribution that would occur at launch.

\subsection{Bias proxies from interference simulations}

Public logs usually contain outcomes under one historical assignment process, so they do not reveal the bias each candidate design would have under alternative interference mechanisms. We therefore use the public datasets to calibrate semi-synthetic interference simulations. Each calibrated case generates a platform panel with baseline outcome \(Y_i^0\), user clusters, shared-budget groups, and repeated periods. Conditional on an assignment vector \(Z\), outcomes are simulated as
\[
Y_i(Z)
=
Y_i^0
+
\tau Z_i
+
\gamma_g^\dagger(\theta)\,\bar Z_{\mathrm{graph}(i),t}
+
\gamma_b^\dagger(\theta)\,\bar Z_{\mathrm{budget}(i),t}
+
\lambda^\dagger(\theta)\,Z_{i,t-1}
+
\varepsilon_i .
\]
Here \(\bar Z_{\mathrm{graph}(i),t}\) is treated exposure in the relevant graph, producer, item, or cluster neighborhood, \(\bar Z_{\mathrm{budget}(i),t}\) is treated exposure in the relevant budget or shared-resource group, and \(Z_{i,t-1}\) is previous-period treatment for the same user.

The notation separates two roles of the ambiguity-grid coordinates. The unit-free mechanism parameters
$
\theta=(\gamma_g,\gamma_b,\lambda,\ell)
$
define the exposure mechanism used by the selector and enter the empirical geometry score in Eq.~\eqref{eq:empirical-exposure-geometry}. The outcome simulation maps the same grid point into case-specific outcome-scale strengths
$
\gamma_g^\dagger(\theta),
$ $
\gamma_b^\dagger(\theta),
$ $
\lambda^\dagger(\theta).
$
In the implementation, this calibration is
\[
\gamma_g^\dagger(\theta)
=
s_{\mathrm{spill}}\rho_g\frac{\gamma_g}{0.3},
\qquad
\gamma_b^\dagger(\theta)
=
s_{\mathrm{spill}}\rho_b\frac{\gamma_b}{0.5},
\qquad
\lambda^\dagger(\theta)
=
s_{\mathrm{carry}}\frac{\lambda}{0.2}.
\]
The constants \(s_{\mathrm{spill}}\) and \(s_{\mathrm{carry}}\) are case-level spillover and carryover scales calibrated from the public data. The weights \(\rho_g\) and \(\rho_b\) split spillover strength across graph-side and budget-side exposure, with \(\rho_g+\rho_b=1\).

Under full launch, direct treatment, graph exposure, budget exposure, and lagged exposure all equal one. The simulated full-launch effect is therefore
$
\tau
+
\gamma_g^\dagger(\theta)
+
\gamma_b^\dagger(\theta)
+
\lambda^\dagger(\theta).
$
Because this target is known inside the semi-synthetic layer, we can replay each candidate design and estimate its exposure mismatch, assignment-unit variance, MDE, contamination or carryover risk, and robust design regret. When known propensities are available, support loss is handled separately through the design-score components rather than by changing the outcome-scale interference coefficients. This keeps the simulation logic aligned with the selector in the sense that ambiguity grid controls exposure mechanisms, while the robust risk aggregates geometry, variance, power, contamination, cost, and estimand-mismatch penalties.

\begin{theorem}[Transport-based bias bound under exposure-response smoothness]
\label{thm:geometry-bias}
Suppose \(P_d^\theta\) and \(P_\star^\theta\) have finite first moments on \((\M,\rho)\). Suppose there exists a measurable exposure-response function \(r_\theta:\M\to\R\) that is \(L_\theta\)-Lipschitz and satisfies
\[
\tau_d(\theta)
=
\int r_\theta(\phi)\,dP_d^\theta(\phi),
\qquad
\left|
\tau^\star(\theta)
-
\int r_\theta(\phi)\,dP_\star^\theta(\phi)
\right|
\le
E_d(\theta).
\]
If \(\widehat\tau_d\) is unbiased for \(\tau_d(\theta)\), then
\begin{equation}
\label{eq:transport-bias-bound}
|B_d(\theta)|
\le
L_\theta G_d(\theta)+E_d(\theta).
\end{equation}
\end{theorem}

The theorem gives the framework its computable core. The bound in Eq.~\eqref{eq:transport-bias-bound} is a transport-based bias bound under exposure-response smoothness. The Wasserstein term in Eq.~\eqref{eq:exposure-geometry} measures how much probability mass must be transported to turn the exposure distribution created by an experiment into the exposure distribution expected at launch, while \(L_\theta\) measures how quickly outcomes can change as exposure changes. In graph experiments, \(G_d(\theta)\) can be computed from treated-neighbor exposure and cluster boundary mass. In ads experiments, it can be computed from advertiser budget pressure, inventory competition, or pacing-state exposure. In switchbacks, it can include current treatment, lagged treatment, daypart, and washout features. The result also clarifies why no single design is expected to dominate. A design can be close to launch exposure in one geometry and far away in another.

\begin{anchorresult}
\begin{theorem}[Minimax optimality of exposure geometry]
\label{thm:geometry-necessity}
Let \(P\) and \(Q\) be probability measures with finite first moments on the Polish metric space \((\M,\rho)\). For \(L>0\), define
\[
\mathcal F_L
=
\{r:\M\to\R:\ |r(x)-r(y)|\le L\rho(x,y)\ \text{for all }x,y\in\M\}.
\]
Then
\begin{equation}
\label{eq:minimax-geometry}
\sup_{r\in\mathcal F_L}
\left|
\int r\,dP-\int r\,dQ
\right|
=
L\Wass(P,Q).
\end{equation}
Consequently, under the representation in Theorem~\ref{thm:geometry-bias},
\[
\sup_{\substack{r_\theta\in\mathcal F_{L_\theta}\\ |\Delta_E|\le E_d(\theta)}}
\left|
\int r_\theta\,dP_d^\theta
-
\int r_\theta\,dP_\star^\theta
-
\Delta_E
\right|
=
L_\theta G_d(\theta)+E_d(\theta).
\]
\end{theorem}
\end{anchorresult}

Equation~\eqref{eq:minimax-geometry} gives the geometry term a minimax interpretation. The analyst does not know the true exposure-response function before the experiment is run, so the theorem asks for the largest possible bias over all response functions that are consistent with the assumed smoothness constraint. The Lipschitz condition rules out arbitrarily discontinuous responses, but it still allows the response to change at the maximum permitted rate along the direction in which the experimental exposure distribution differs from the launch exposure distribution. In that least favorable case, the bias is exactly proportional to \(W_1(P_d^\theta,P_\star^\theta)\). \textit{Thus the Wasserstein penalty is not only an upper bound; it is the smallest uniform penalty that can protect against all \(L_\theta\)-Lipschitz exposure-response functions}. If a design-choice rule ignores \(G_d(\theta)\), it can rank a design as safe even though there exists a smooth response function under which the design has large launch bias. The theorem does not claim that every real application attains this worst case. Rather, it says that before stronger structural knowledge about the response surface is available, the Wasserstein exposure-distance term is the correct robust object for guarding against exposure-induced bias.

The selector uses this minimax result to justify \(G_d(\theta)\) as the geometry component in the planning score. The implemented score keeps the components separate, because exposure mismatch and residual estimand mismatch are measured differently in logs. Let \(\bar G_d,\bar V_d,\bar M_d,\bar C_d,\bar O_d,\bar E_d\) denote the population analogues of the case-level normalized component scores. For example, \(\bar G_d\) is obtained by scaling the exposure distance \(G_d\) to the same unit-free planning scale as the other components. The geometry-aware planning risk is
\begin{equation}
\label{eq:geometry-aware-risk}
\mathcal R_T^{\mathrm{geo}}(d;\theta)
=
w_g\bar G_d(\theta)
+w_v\bar V_d(\theta)
+w_m\bar M_d(T;\theta)
+w_c\bar C_d(\theta)
+w_o\bar O_d
+w_e\bar E_d(\theta) .
\end{equation}
If an application has a calibrated exposure-response sensitivity \(L_\theta\), that sensitivity can be absorbed into the geometry component before normalization. The experiments use the conservative public-data version in which \(G_d(\theta)\) is populated by the paired exposure-distance proxy described in Algorithm~\ref{alg:robust-selector}, and \(E_d(\theta)\) is the separate estimand-mismatch score.

\subsection{Finite catalogs as design approximations}

Platforms do not optimize over every conceivable randomization scheme. They choose from a finite operational catalog, such as candidate cluster granularities, switchback block lengths, budget-split fractions, and saturation levels. This is not merely a convenience if the catalog covers the relevant design geometry.

\begin{proposition}[Finite catalog approximation]
\label{prop:catalog-approximation}
Let \((\mathfrak D,\rho_{\mathfrak D})\) be a continuous feasible design class and let \(\D_{\mathrm{cat}}\subset \mathfrak D\) be an \(\eta\)-net, so for every \(d\in\mathfrak D\) there exists \(d'\in\D_{\mathrm{cat}}\) with \(\rho_{\mathfrak D}(d,d')\le \eta\). If the robust geometry-aware risk
\[
\mathcal R_{\mathrm{rob}}(d)
=
\sup_{\theta\in\U}\mathcal R_T^{\mathrm{geo}}(d;\theta)
\]
is \(L_{\mathfrak D}\)-Lipschitz in \(d\), then
\[
\min_{d\in\D_{\mathrm{cat}}}\mathcal R_{\mathrm{rob}}(d)
\le
\inf_{d\in\mathfrak D}\mathcal R_{\mathrm{rob}}(d)
+L_{\mathfrak D}\eta .
\]
\end{proposition}

This result explains why a finite menu of designs can be enough for practical experiment planning. In principle, a platform could consider infinitely many variants of a design, such as every possible switchback block length, every possible cluster granularity, every possible budget-split fraction, or every possible saturation level. In practice, only a finite set of these designs can be engineered, monitored, and approved before launch. Proposition~\ref{prop:catalog-approximation} says that this restriction is acceptable when the catalog is dense enough in the design space. If every feasible design is within distance \(\eta\) of some catalog design, and robust risk does not change too abruptly with the design, then the best catalog design is within \(L_{\mathfrak D}\eta\) risk units of the best design in the larger continuous class.

The intuition is that a catalog is good enough when adding more nearby designs would not materially change the risk decision. If the risk surface is smooth, two very similar switchback lengths or two very similar saturation levels should have similar exposure geometry, similar assignment-unit variance, and similar operational risk. The finite catalog only fails when the grid is too coarse relative to how quickly the risk surface changes. The implementation checks this logic in the catalog-approximation diagnostic in Figure~\ref{fig:appendix-theory-checks}(c). The code constructs a smooth one-dimensional robust-risk surface, evaluates increasingly fine finite catalogs, and compares the observed catalog approximation gap with the Lipschitz-net upper bound \(L_{\mathfrak D}\eta\). As the catalog size increases, the net radius \(\eta\) shrinks, and the observed gap remains below the bound. This does not prove that every production catalog is automatically sufficient; rather, it shows how the paper's finite-catalog assumption can be audited by comparing catalog resolution with the variation in exposure geometry, assignment-unit variance, and planning risk.

\subsection{The bias-variance design regime}

Exposure geometry is only one side of the design problem. Designs that better match launch exposure often use larger assignment units, which increases variance and MDE. The following regime inequality makes this tradeoff explicit.

\begin{proposition}[Design-regime threshold]
\label{prop:regime-threshold}
Consider two feasible designs \(d_1,d_2\) under a mechanism \(\theta_\gamma\) whose interference strength is indexed by \(\gamma\ge0\). Suppose \(\bar G_d(\theta_\gamma)\le \gamma g_d\), where \(g_d\) is the design's normalized exposure-geometry coefficient. Ignore common terms shared by the two designs and compare the upper-bound surrogate
\[
\widetilde{\mathcal R}_T(d;\theta_\gamma)
=
w_g\gamma g_d
+w_v\bar V_d+w_m\bar M_d(T)+w_c\bar C_d+w_o\bar O_d+w_e\bar E_d .
\]
If \(g_2<g_1\), then design \(d_2\) has smaller surrogate risk than \(d_1\) whenever
\[
\gamma
>
\frac{
 w_v(\bar V_2-\bar V_1)+w_m(\bar M_2(T)-\bar M_1(T))+w_c(\bar C_2-\bar C_1)+w_o(\bar O_2-\bar O_1)+w_e(\bar E_2-\bar E_1)
}{
 w_g(g_1-g_2)
},
\]
provided the denominator is positive.
\end{proposition}

Proposition~\ref{prop:regime-threshold} is the paper's regime map because it shows how the preferred design changes as the interference-strength parameter \(\gamma\) increases. In the surrogate risk, the term \(w_g\gamma g_d\) is the part of the risk that grows with interference. The coefficient \(g_d\) measures how exposed design \(d\) is to geometry mismatch as interference becomes stronger. A design such as user randomization may have a larger \(g_d\), because its experimental exposure distribution can be far from launch exposure under spillovers, but it may have smaller non-geometry terms such as \(\bar V_d\), \(\bar M_d(T)\), \(\bar C_d\), \(\bar O_d\), and \(\bar E_d\). A more protective design, such as cluster randomization or a switchback, may have a smaller \(g_d\) but larger variance, MDE, cost, or carryover terms.

The proposition compares two designs \(d_1\) and \(d_2\) with \(g_2<g_1\). Design \(d_2\) has better exposure geometry under interference, but it may pay an additional non-geometry cost. The threshold defined on $\gamma$ is exactly the point at which the geometry advantage of \(d_2\) outweighs its additional variance, power, contamination, operational-cost, and estimand-mismatch burden. Below this threshold, the lower-cost or higher-power design \(d_1\) can be preferred. Above it, the interference-protective design \(d_2\) becomes preferred because geometry mismatch dominates the planning risk.

In production, this threshold is not usually known in closed form. The terms \(\bar V_d\), \(\bar M_d(T)\), \(\bar C_d\), \(\bar O_d\), \(\bar E_d\), and \(g_d\) must be estimated or specified from logs, design cards, and stress tests. The empirical regime-transition plot (Figure \ref{fig:regime-transition}) uses this logic operationally as it sweeps the interference-strength parameter, recomputes the normalized risk components for each design, and shows where the selected design changes from user randomization to mixed randomization, cluster randomization, and switchbacks.

\subsection{Geometry-aware robust selector}

Using the geometry-aware planning risk in Eq.~\eqref{eq:geometry-aware-risk}, the empirical selector replaces each component by an estimate. Its estimated risk surface is
\begin{equation}
\label{eq:empirical-geometry-risk}
\widehat{\mathcal R}_T^{\mathrm{geo}}(d;\theta)
=
w_g\widetilde G_d(\theta)
+w_v\widetilde V_d(\theta)
+w_m\widetilde M_d(T;\theta)
+w_c\widetilde C_d(\theta)
+w_o\widetilde O_d
+w_e\widetilde E_d(\theta) .
\end{equation}
The tildes denote the normalized empirical component scores produced by Algorithm~\ref{alg:robust-selector}.
The empirical selector chooses
\begin{equation}
\label{eq:empirical-selector}
\widehat d_T
\in
\argmin_{d\in\D_{\mathrm{cat}}}
\sup_{\theta\in\U}
\widehat{\mathcal R}_T^{\mathrm{geo}}(d;\theta).
\end{equation}

\begin{theorem}[Robust selector with margin and shortlist certification]
\label{thm:geometry-robust-selector}
Let \(\D_{\mathrm{cat}}\) and \(\U\) be finite. Suppose that, uniformly over \(d\in\D_{\mathrm{cat}}\) and \(\theta\in\U\),
\[
|\widetilde G_d(\theta)-\bar G_d(\theta)|\le \epsilon_G,\quad
|\widetilde E_d(\theta)-\bar E_d(\theta)|\le \epsilon_E,
\]
\[
|\widetilde V_d(\theta)-\bar V_d(\theta)|\le \epsilon_V,\quad
|\widetilde M_d(T;\theta)-\bar M_d(T;\theta)|\le \epsilon_M,\quad
|\widetilde C_d(\theta)-\bar C_d(\theta)|\le \epsilon_C,\quad
|\widetilde O_d-\bar O_d|\le \epsilon_O .
\]
Define
\begin{equation}
\label{eq:selector-error-budget}
\epsilon_T
=
w_g\epsilon_G
+w_v\epsilon_V+w_m\epsilon_M+w_c\epsilon_C+w_o\epsilon_O+w_e\epsilon_E .
\end{equation}
Let \(d_T^\star\) minimize \(\sup_{\theta\in\U}\mathcal R_T^{\mathrm{geo}}(d;\theta)\) over \(\D_{\mathrm{cat}}\), and let \(\widehat d_T\) minimize \(\sup_{\theta\in\U}\widehat{\mathcal R}_T^{\mathrm{geo}}(d;\theta)\). Then
\begin{equation}
\label{eq:selector-excess-risk}
\sup_{\theta\in\U}\mathcal R_T^{\mathrm{geo}}(\widehat d_T;\theta)
-
\sup_{\theta\in\U}\mathcal R_T^{\mathrm{geo}}(d_T^\star;\theta)
\le
2\epsilon_T .
\end{equation}
Let
$
Q(d)=\sup_{\theta\in\U}\mathcal R_T^{\mathrm{geo}}(d;\theta),
$ $
Q^\star=\min_{d\in\D_{\mathrm{cat}}}Q(d),
$ and let \(\D^\star=\{d\in\D_{\mathrm{cat}}:Q(d)=Q^\star\}\). If the separation margin
$
\Delta_{\mathrm{sep}}
=
\min_{d\in\D_{\mathrm{cat}}\setminus\D^\star}\{Q(d)-Q^\star\}
$
is well-defined and satisfies \(\Delta_{\mathrm{sep}}>2\epsilon_T\), then \(\widehat d_T\in\D^\star\). If \(\D_{\mathrm{cat}}=\D^\star\), every design is robust-optimal. More generally, for any shortlist tolerance \(\kappa\ge0\), define the empirical shortlist
\begin{equation}
\label{eq:empirical-shortlist}
\widehat{\mathcal S}_\kappa
=
\left\{
d\in\D_{\mathrm{cat}}:
\widehat Q(d)
\le
\min_{d'\in\D_{\mathrm{cat}}}\widehat Q(d')+\kappa
\right\},
\qquad
\widehat Q(d)=\sup_{\theta\in\U}\widehat{\mathcal R}_T^{\mathrm{geo}}(d;\theta).
\end{equation}
Then, every \(d\in\widehat{\mathcal S}_\kappa\) satisfies \(Q(d)-Q^\star\le\kappa+2\epsilon_T\), and at least one robust-optimal design belongs to \(\widehat{\mathcal S}_{2\epsilon_T}\).
\end{theorem}

Theorem~\ref{thm:geometry-robust-selector} turns the selector into a finite-sample decision rule. The quantities in Eq.~\eqref{eq:selector-error-budget} measure how uncertain the estimated planning-risk surface is. For example, \(\epsilon_G\) measures uncertainty in the exposure-geometry score, \(\epsilon_V\) measures uncertainty in the assignment-unit variance estimate, \(\epsilon_M\) measures uncertainty in the MDE calculation, \(\epsilon_C\) measures uncertainty in contamination or carryover modeling, \(\epsilon_O\) measures uncertainty in the operational-cost score, and \(\epsilon_E\) measures uncertainty in residual estimand mismatch. The weighted sum of these component errors is \(\epsilon_T\), the planning-error budget for the robust-risk comparison. The theorem says that the selector should behave differently depending on how separated the designs are. If one design has a robust-risk advantage that is larger than the planning uncertainty, then the empirical minimizer in Eq.~\eqref{eq:empirical-selector} recovers the robust-optimal design. In that case, the data and design assumptions support a single recommendation. If several designs have robust risks that are close relative to \(\epsilon_T\), then forcing a unique winner would overstate the evidence. The appropriate output is the shortlist in Eq.~\eqref{eq:empirical-shortlist}. Every design in that shortlist has robust regret bounded by \(\kappa+2\epsilon_T\), so the shortlist is not merely descriptive; it is a certified set of designs that cannot be safely separated at the current planning resolution. In this sense, the selector returns a statistical decision object consisting of either a justified design choice when the risk margin is clear, or an uncertainty shortlist when the design comparison is intrinsically unresolved.

% \subsection{Planning MDE under assignment-unit variance}

% For a design \(d\), let \(A_{dt}\) denote the assignment unit observed at time or block \(t\). Examples include user, cluster, advertiser-budget universe, region-hour, or producer cluster. The detectable effect depends on the number of independent assignment units, not just the number of logged rows. A simple planning calculation is
% \begin{equation}
% \label{eq:planning-mde}
% \MDE_d(T)
% =
% (z_{1-\alpha/2}+z_{1-\beta})
% \sqrt{\frac{2\sigma_d^2}{N_d(T)}},
% \end{equation}
% where \(\sigma_d^2\) is the assignment-unit variance and \(N_d(T)\) is the number of effective assignment units accumulated over duration \(T\). This expression gives a common scale for comparing row-level, cluster, switchback, and budget-split designs. It also makes clear why a design can have many rows and still have weak power if it has few independent assignment units.

\section{Empirical Results}

The empirical section asks whether the robust design-selection object can be populated from real logs before an online experiment is launched. Across the main-body experiments, public datasets calibrate outcome scale, triggered exposure or repeated-exposure structure, logged propensities when available, assignment-unit structure, and semi-synthetic launch effects. The interference layer is semi-synthetic because public logs do not reveal the production counterfactual under every candidate randomization design. This choice keeps the target launch effect known while anchoring each case to public logs through the quantities actually available in that dataset. Appendix~\ref{sec:appendix-empirical} reports the detailed per-case diagnostics, the Open Bandit randomized data baseline, the MovieLens graph case, and theorem-level stress tests.

\begin{table}[t]
\centering
\small
\caption{Real-data cases used in the main empirical analysis. Each case estimates the same robust design-selection object under a different uncertain exposure mechanism.}
\label{tab:data-plan}
\setlength{\tabcolsep}{3pt}
\renewcommand{\arraystretch}{1.08}
\begin{tabular}{>{\raggedright\arraybackslash}p{0.16\linewidth}>{\raggedright\arraybackslash}p{0.18\linewidth}>{\raggedright\arraybackslash}p{0.23\linewidth}>{\raggedright\arraybackslash}p{0.34\linewidth}}
\toprule
Dataset & Domain & Mechanism stress & Role in the paper \\
\midrule
Criteo Uplift & Ads incrementality & Budget and inventory spillover & Calibrates an ads lift problem in which user-level assignment can miss shared-budget exposure. \\
Open Bandit & Recommendation bandits & Known propensities and adaptive logging & Tests whether the selector changes when logged action support is thin under adaptive behavior-policy propensities. \\
KuaiRand & Member-experience recommendation & Sequential exposure and carryover & Tests designs under repeated member exposure and item-side recommendation state. \\
\bottomrule
\end{tabular}
\end{table}

Each empirical case instantiates Algorithm~\ref{alg:robust-selector}. For every design \(d\) and mechanism \(\theta\), the framework estimates the empirical exposure-distance proxy for \(G_d(\theta)\), assignment-unit variance \(V_d(\theta)\), planning MDE \(M_d(T;\theta)\), carryover or contamination risk \(C_d(\theta)\), operational cost \(O_d\), and estimand mismatch \(E_d(\theta)\). Before aggregation, each component is divided by the largest absolute value of that component among the candidate designs for the same empirical case. The reported risk is therefore a dimensionless weighted sum of normalized components, not an outcome-scale effect size. The values are useful for within-case design ranking and shortlist construction, however they should not be interpreted as lift, bias, or percentage loss. Additionally, cross-case comparisons should be read qualitatively. The empirical selector then returns the lowest-risk design together with the \(2\epsilon_T\) uncertainty shortlist from Theorem~\ref{thm:geometry-robust-selector}. In the implementation, \(\epsilon_T\) is set to ten percent of the best estimated robust risk for the case. This makes the shortlist conservative without masking risk ordering.

\begin{table}[t]
\centering
\small
\caption{Main robust design-selection results. The selected design minimizes estimated worst-case planning risk over the calibrated ambiguity set. Risk is a dimensionless weighted sum after within-case normalization of the risk components. The shortlist contains designs whose estimated robust risk is within the theorem-guided uncertainty band.}
\label{tab:main-results}
\setlength{\tabcolsep}{3pt}
\renewcommand{\arraystretch}{1.08}
\begin{tabular}{>{\raggedright\arraybackslash}p{0.18\linewidth}>{\raggedright\arraybackslash}p{0.27\linewidth}>{\raggedright\arraybackslash}p{0.18\linewidth}>{\raggedleft\arraybackslash}p{0.10\linewidth}>{\raggedright\arraybackslash}p{0.19\linewidth}}
\toprule
Case & Exposure uncertainty & Selected design & Risk & Tolerance shortlist \\
\midrule
Criteo ads & Budget spillover calibrated from triggered exposure and visit lift & User randomization & 1.295 & User \\
Open Bandit \texttt{bts/men} & Adaptive logged propensities with 5.17\% IPS effective-sample share & Switchback & 2.105 & Switchback, cluster, budget split \\
KuaiRand & Sequential recommendation exposure and carryover & Cluster randomization & 2.240 & All six catalog designs \\
\bottomrule
\end{tabular}
\end{table}

Here Table \ref{tab:data-plan} introduces the relevant datasets, and Table~\ref{tab:main-results} presents the main results of the analysis. The selected design changes across datasets. Criteo favors user randomization, Open Bandit \texttt{bts/men} favors switchbacks, and KuaiRand favors cluster randomization. This variation is the central empirical point. Instead of learning a universal design, the robust selector is converting the calibrated exposure mechanism, assignment-unit variance, contamination risk, and operational constraints into a case-specific decision. Theorem~\ref{thm:geometry-robust-selector} explains the shortlist behavior. Criteo has a clear risk separation and returns only user randomization. Open Bandit keeps switchback, cluster randomization, and budget split. KuaiRand has a flatter risk surface and retains all six catalog designs as unresolved validation candidates.

The regime-transition diagnostic (from synthetic data) in Figure~\ref{fig:regime-transition} is the empirical counterpart of Proposition~\ref{prop:regime-threshold}. It isolates the exposure mechanism while keeping the same design catalog and risk components. User randomization is best when interference is weak because variance and operating cost dominate. Mixed randomization becomes best when moderate spillover makes pure user assignment too contaminated but full cluster or time assignment is still costly. Cluster randomization wins once spillover is strong enough that exposure geometry dominates power. Switchbacks win when carryover becomes the main threat. The three real-data cases then instantiate different parts of this decision logic (as seen in Table~\ref{tab:main-results}). Criteo remains in a regime where user randomization is robust, Open Bandit \texttt{bts/men} moves toward temporal blocking under adaptive support stress, and KuaiRand favors clustered assignment under repeated recommendation exposure.

\begin{figure}[t]
\centering
\includegraphics[width=0.65\linewidth]{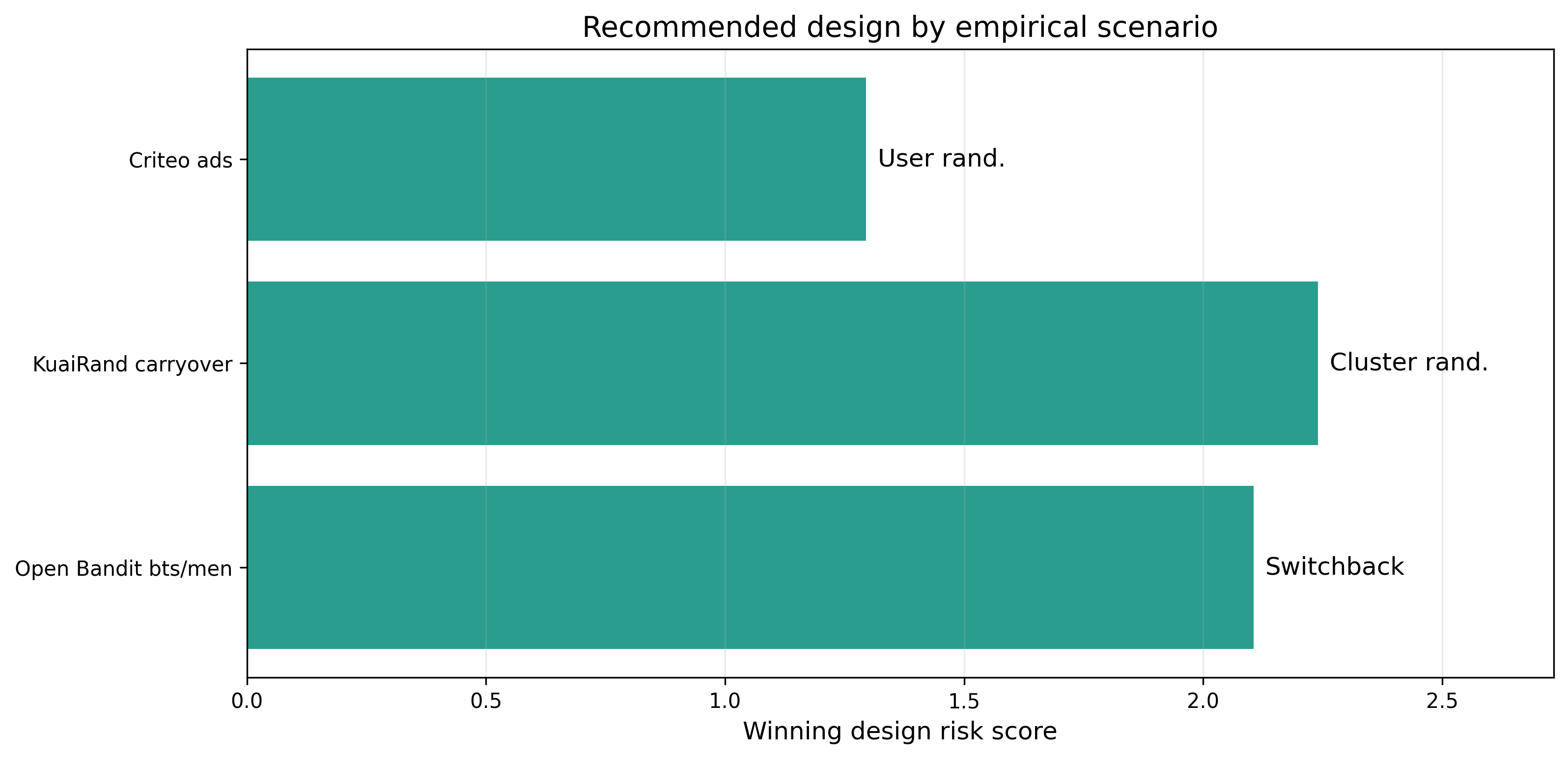}
\caption{\small Cross-domain robust design recommendations. The figure compares the main real-data cases using the same design catalog and planning-risk components. The selected design changes across domains: user randomization for Criteo, switchbacks for Open Bandit \texttt{bts/men}, and cluster randomization for KuaiRand. Figure~\ref{fig:main-risk-rankings} gives the corresponding robust-risk rankings, while Appendix~\ref{sec:appendix-empirical} reports frontier plots and supporting cases.}
\label{fig:cross-domain-results}
\end{figure}

\begin{figure}[t]
\centering
\begin{subfigure}{0.48\linewidth}
\centering
\includegraphics[width=\linewidth]{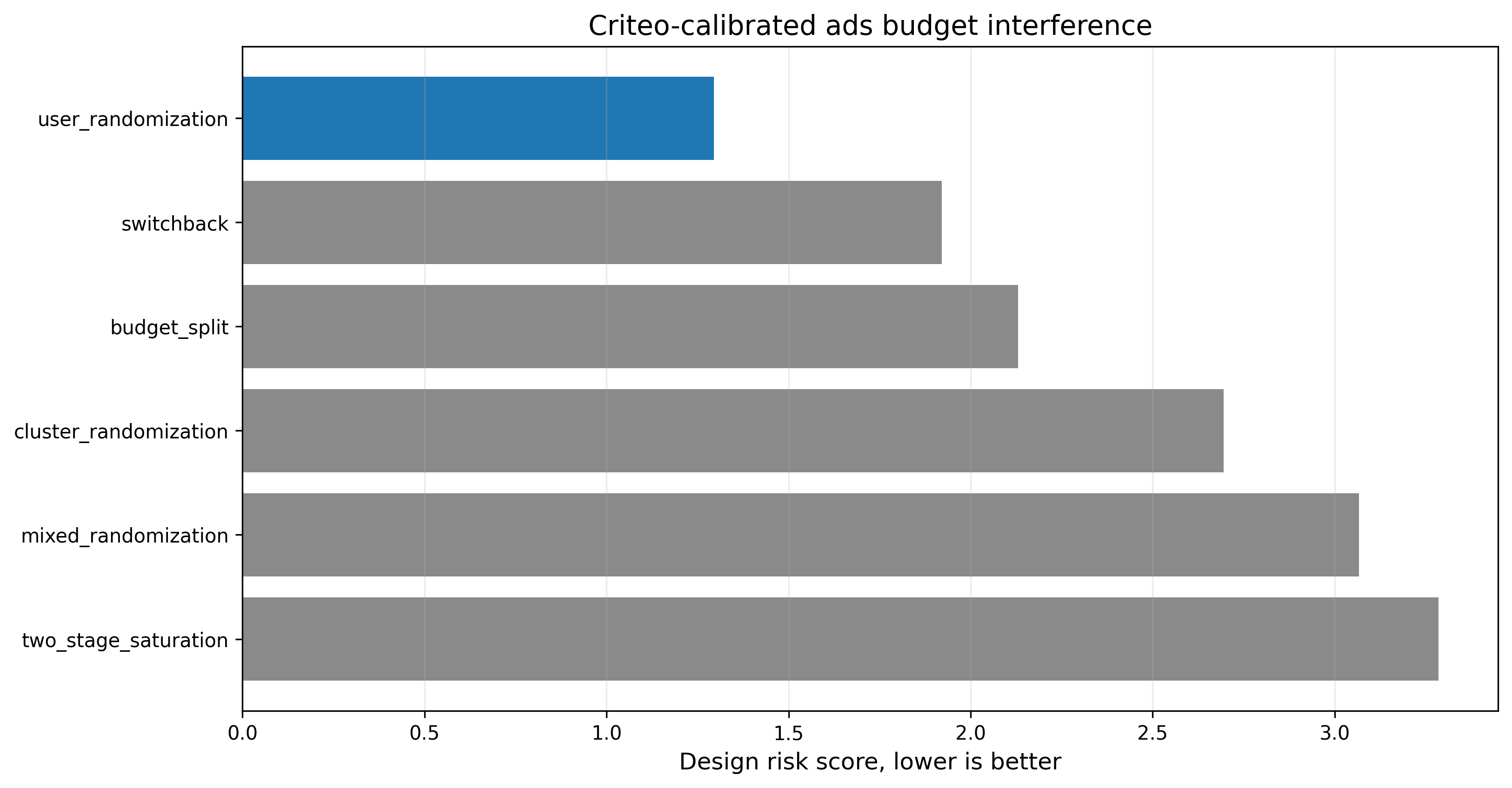}
\caption{Criteo ads}
\end{subfigure}
\hfill
\begin{subfigure}{0.48\linewidth}
\centering
\includegraphics[width=\linewidth]{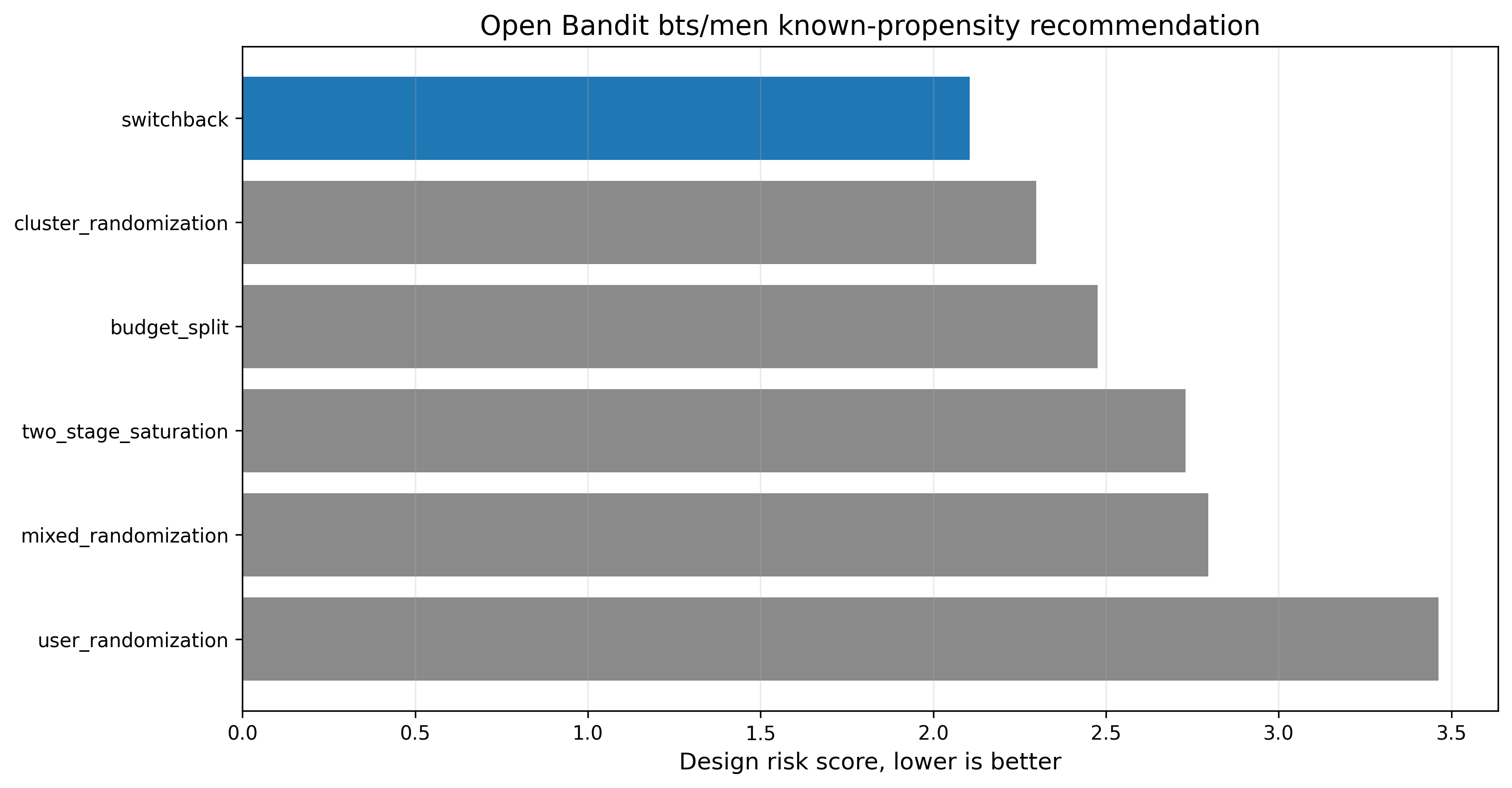}
\caption{Open Bandit \texttt{bts/men}}
\end{subfigure}

\vspace{0.5em}

\begin{subfigure}{0.58\linewidth}
\centering
\includegraphics[width=\linewidth]{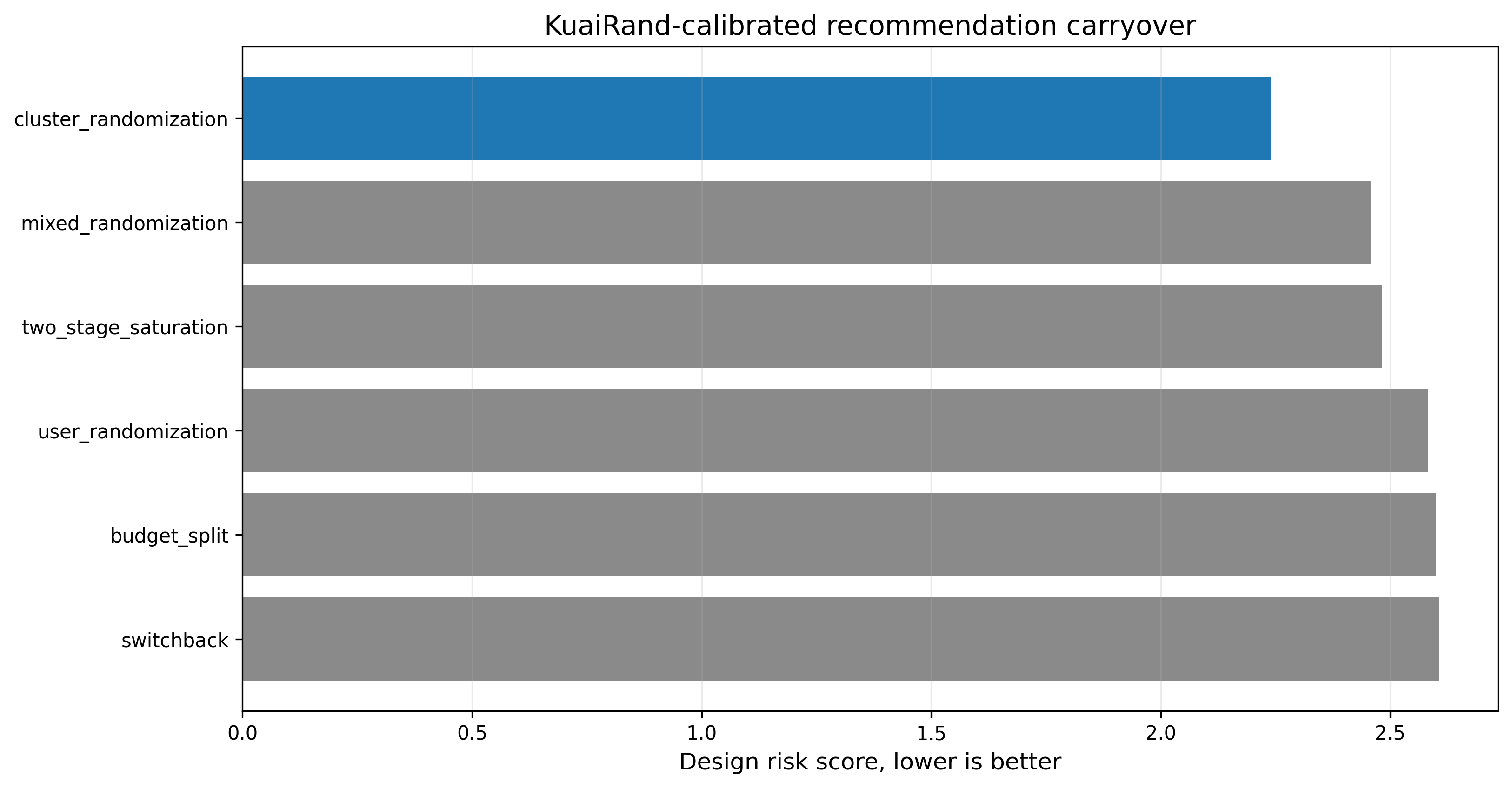}
\caption{KuaiRand}
\end{subfigure}
\caption{\small Robust-risk rankings for the three main empirical cases. Each panel ranks the six implementable designs by estimated worst-case planning risk over the calibrated ambiguity set. The figure explains the decisions in Table~\ref{tab:main-results}: Criteo has a separated user-randomization leader, Open Bandit \texttt{bts/men} has a closer contest between switchbacks and clustered alternatives, and KuaiRand has a flat risk surface that keeps all six designs in the tolerance shortlist.}
\label{fig:main-risk-rankings}
\end{figure}

\paragraph{Ads budget spillover.}
The Criteo analysis uses 300,000 logged examples to calibrate a binary visit outcome with a visit lift of \(0.00298\) and a conversion lift of \(0.00148\). The semi-synthetic layer makes advertiser budget spillover the uncertain mechanism. In this calibrated case, user randomization has the lowest robust risk, \(1.295\), and is the only design in the tolerance shortlist. This does not imply that budget spillovers are irrelevant. Rather, the estimated tradeoff says that the additional protection offered by cluster, budget-split, or mixed designs is not worth their higher exposure mismatch, MDE, or operational burden under the Criteo-calibrated ambiguity set. Detailed Criteo risk and exposure-variance plots are shown in Appendix~\ref{sec:appendix-empirical}.

\paragraph{Known propensities and adaptive logging.}
The Open Bandit case isolates a different source of uncertainty. The \texttt{random/men} slice has constant logging propensities and an IPS effective-sample share of 1.00, so it behaves like a clean randomized baseline. The \texttt{bts/men} slice has adaptive behavior-policy propensities, with propensities ranging from \(0.00006\) to \(0.594\) and an IPS effective-sample share of \(0.0517\). The paper uses \texttt{bts/men} in the main body because it is the harder design-selection problem. The implementation converts the low effective-sample share into a design-sensitive support-stress adjustment added to contamination/carryover and estimand mismatch. The selector chooses switchbacks with robust risk \(2.105\), followed by cluster randomization and budget split in the tolerance shortlist. This implies that under the calibrated recommendation and carryover ambiguity set plus severe support stress, switchbacks have the lowest robust risk, while the shortlist acknowledges that cluster and budget-split designs remain plausible under alternative exposure channels. The randomized Open Bandit slice is reported in Appendix~\ref{sec:appendix-empirical} as a known-propensity baseline.

\paragraph{Member-experience carryover.}
The KuaiRand experiment uses 250,000 short-video recommendation events from 5,818 users and 7,517 items. The click rate is 17.7\%, which provides enough outcome signal to compare designs, while repeated exposure creates a plausible carryover channel. Through experiments it was found that cluster randomization minimizes robust risk at \(2.240\). The tolerance shortlist is deliberately long and retains all six designs, because mixed randomization, two-stage saturation, user randomization, budget split, and switchbacks are all within the uncertainty band. This is exactly the role of the robust selector. It distinguishes the current lowest-risk recommendation from designs that remain plausible when the exposure mechanism is not fully known.

Taken together (Figure \ref{fig:main-risk-rankings}), the regime reversal and the three real-data cases support the paper's central formulation. The empirical contribution is a reproducible workflow for turning uncertain exposure mechanisms into a robust design-selection decision. The appendix provides additional stress tests by separating theorem-level diagnostics from domain-level results. It reports supporting risk rankings, exposure-variance frontiers, the Open Bandit randomized baseline, the MovieLens graph-interference frontier, the transport bias bound, the minimax tightness of the exposure geometry penalty, catalog approximation, and assignment-unit MDE scaling.

\section{Conclusion}

Interference changes experiment planning from estimator selection to design selection. This paper formulates that choice as robust optimization over uncertain exposure mechanisms. Exposure geometry supplies the bias proxy, finite catalogs supply the implementable action set, and the robust selector converts estimated risk components into either a design recommendation or an uncertainty shortlist. Across Criteo, Open Bandit, and KuaiRand datasets, the same workflow produces different recommendations from different exposure uncertainties, while the appendix verifies that the theoretical components behave as claimed. The resulting message is narrow. Platforms should not choose a randomization design by habit when spillovers, carryover, budgets, or producer exposure may matter. Instead, they should specify the plausible exposure mechanisms, estimate the geometry and assignment-unit tradeoffs, and select the design that is robust to that uncertainty.

Future work is expected to focus on learning ambiguity sets directly from pre-experiment telemetry, calibrate planning-risk weights from organizational loss functions, and connect the selector to adaptive experiment governance after launch. The framework also suggests a practical systems direction. Experiment platforms can expose robust design-selection diagnostics before assignment begins, so that design choice becomes an auditable decision rather than an informal pre-analysis convention.

\paragraph{Code availability.}
The code and notebooks supporting the empirical analysis are available at \url{https://github.com/p-shekhar/Online-experiment-design.git}.

\bibliographystyle{plainnat}
\bibliography{references}

\appendix

\section{Additional Empirical Diagnostics}
\label{sec:appendix-empirical}

The appendix separates three kinds of supporting evidence. The first kind checks the theory under controlled conditions where the relevant quantities are known or can be swept directly. The second kind reports detailed case-level plots that are summarized in the main empirical results (Table \ref{tab:main-results}) but not shown there because of space. The third kind adds supporting cases, including the Open Bandit randomized baseline and a MovieLens graph-recommendation substrate. Table~\ref{tab:theory-empirical-map} gives the bridge between formal results and empirical artifacts.

\begin{table}[t]
\centering
\scriptsize
\caption{Relationship between theoretical results and empirical diagnostics. Main-body rows use real datasets to populate the robust selector; appendix rows isolate theorem-level behavior.}
\label{tab:theory-empirical-map}
\setlength{\tabcolsep}{2.5pt}
\renewcommand{\arraystretch}{1.08}
\begin{tabular}{>{\raggedright\arraybackslash}p{0.22\linewidth}
                >{\raggedright\arraybackslash}p{0.27\linewidth}
                >{\raggedright\arraybackslash}p{0.18\linewidth}
                >{\raggedright\arraybackslash}p{0.25\linewidth}}
\toprule
Result & Empirical diagnostic & Figure/table location & Finding \\
\midrule
Lemma~\ref{lem:no-universal-design} and Proposition~\ref{prop:regime-threshold}
& Regime transition, appendix regime-reversal diagnostic, and real-data shortlists
& Fig.~\ref{fig:regime-transition}; Fig.~\ref{fig:appendix-selector-diagnostics}(a); Table~\ref{tab:main-results}
& User, mixed, cluster, and switchback designs win under different mechanisms; the main real-data cases select different designs and retain different tolerance shortlists. \\

Theorem~\ref{thm:geometry-bias}
& Transport-bound stress test
& Fig.~\ref{fig:appendix-theory-checks}(a)
& Constructed exposure-response checks meet the empirical \(L\Wass\) transport bound, confirming the bound and illustrating tight cases. \\

Theorem~\ref{thm:geometry-necessity}
& Minimax exposure-geometry construction
& Fig.~\ref{fig:appendix-theory-checks}(b)
& Lipschitz response functions attain the transport penalty, confirming that the Wasserstein term is not merely a loose heuristic. \\

Proposition~\ref{prop:catalog-approximation}
& Catalog-grid refinement diagnostic
& Fig.~\ref{fig:appendix-theory-checks}(c)
& The catalog approximation gap stays below the Lipschitz-net bound as the design grid is refined. \\

Theorem~\ref{thm:geometry-robust-selector}
& Cross-domain shortlists and synthetic oracle comparison
& Table~\ref{tab:main-results}; Fig.~\ref{fig:main-risk-rankings}; Fig.~\ref{fig:appendix-selector-diagnostics}(b)
& The real-data selector returns tolerance-based shortlists under the same decision rule; in the controlled oracle diagnostic, the selected design equals the oracle robust design and satisfies the excess-risk certificate. \\

Assignment-unit MDE calculation
& MDE grid over duration and effective assignment units
& Fig.~\ref{fig:appendix-theory-checks}(d); Fig.~\ref{fig:appendix-domain-frontiers}
& Detectable effects shrink with duration and effective assignment-unit count, emphasizing why row count alone is not enough under interference. \\
\bottomrule
\end{tabular}
\end{table}

\begin{table}[t]
\centering
\scriptsize
\caption{Supporting appendix cases not used as the main cross-domain recommendations. The Open Bandit randomized slice is a known-propensity baseline, and MovieLens supplies an additional graph-recommendation substrate.}
\label{tab:appendix-supporting-cases}
\setlength{\tabcolsep}{3pt}
\renewcommand{\arraystretch}{1.08}
\begin{tabular}{>{\raggedright\arraybackslash}p{0.24\linewidth}>{\raggedright\arraybackslash}p{0.25\linewidth}>{\raggedright\arraybackslash}p{0.18\linewidth}>{\raggedleft\arraybackslash}p{0.10\linewidth}>{\raggedright\arraybackslash}p{0.16\linewidth}}
\toprule
Case & Role & Selected design & Risk & Tolerance shortlist \\
\midrule
Open Bandit \texttt{random/men} & Randomized known-propensity baseline with IPS ESS share \(1.00\) & Switchback & 2.182 & Switchback, cluster, two-stage, mixed, budget split \\
MovieLens graph & Appendix graph-recommendation substrate with user-item exposure & Cluster randomization & 1.976 & Cluster, budget split, two-stage, mixed \\
\bottomrule
\end{tabular}
\end{table}

The supporting cases in Table~\ref{tab:appendix-supporting-cases} serve two purposes. The Open Bandit randomized slice checks the same recommendation substrate when logged propensities are uniform and support is not the main stressor. Switchbacks remain the lowest-risk design in this slice, but the risk surface is broad enough to retain cluster randomization, two-stage saturation, mixed randomization, and budget split in the shortlist. This comparison sharpens the main-body interpretation of \texttt{bts/men}: the adaptive slice adds severe support stress, while the design recommendation is still governed by the calibrated recommendation and carryover ambiguity set. Figure~\ref{fig:openbandit-propensities} provides the corresponding propensity diagnostic. MovieLens dataset gives a separate user-item graph substrate. It is not used as a main empirical anchor because it is a broader recommendation graph stress test, but it confirms that the same selector can be populated outside the ads, bandit, and short-video cases.

\begin{figure}[t]
\centering
\includegraphics[width=0.74\linewidth]{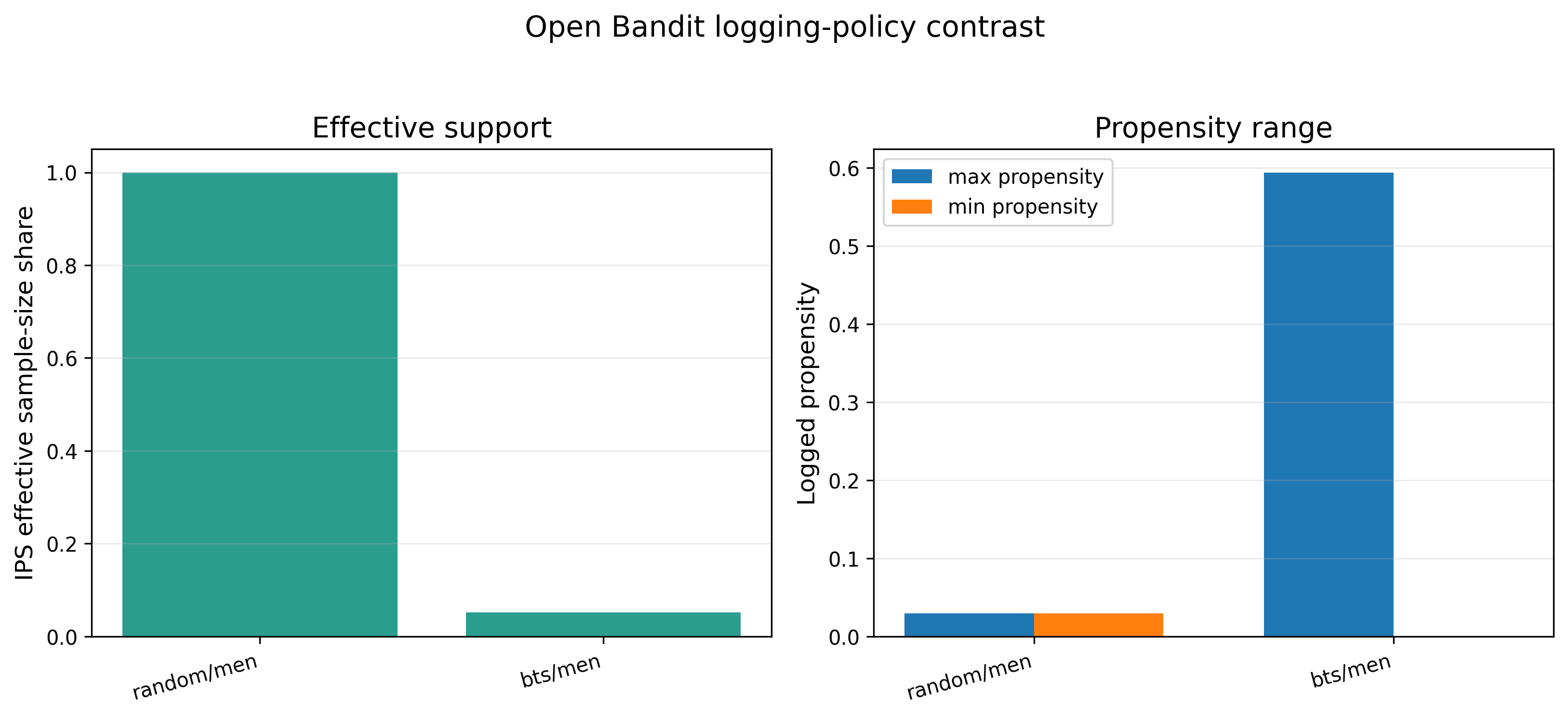}
\caption{\small Open Bandit propensity diagnostics. The randomized slice has constant propensities and full effective sample size, while the adaptive \texttt{bts/men} slice has highly variable propensities and an IPS effective-sample share of 5.17\%. This figure supports the main Open Bandit design-selection result by showing why the adaptive slice is the relevant stress test for known but uneven logging support.}
\label{fig:openbandit-propensities}
\end{figure}

Figure~\ref{fig:appendix-domain-risk-rankings} supplements the main robust-risk rankings by keeping the main cases and the MovieLens appendix case side by side. The Criteo panel shows why the recommendation is decisive. This is because the user randomization is separated from the next design after worst-case normalization. The Open Bandit panel shows a closer contest between switchbacks and cluster randomization, consistent with the main-body interpretation that adaptive support stress changes the preferred design but does not eliminate clustered alternatives. The KuaiRand panel explains the long shortlist based on the risk curve being relatively flat after the leader. The MovieLens panel gives the corresponding graph-interference pattern, where cluster randomization leads but two-stage saturation, budget split, and mixed randomization remain credible alternatives.

\begin{figure}[p]
\centering
\begin{subfigure}{0.48\linewidth}
\centering
\includegraphics[width=\linewidth]{figures/02_criteo_ads_design_risk.png}
\caption{Criteo ads}
\end{subfigure}
\hfill
\begin{subfigure}{0.48\linewidth}
\centering
\includegraphics[width=\linewidth]{figures/03_open_bandit_design_risk.png}
\caption{Open Bandit \texttt{bts/men}}
\end{subfigure}

\vspace{0.6em}

\begin{subfigure}{0.48\linewidth}
\centering
\includegraphics[width=\linewidth]{figures/04_kuairand_design_risk.png}
\caption{KuaiRand}
\end{subfigure}
\hfill
\begin{subfigure}{0.48\linewidth}
\centering
\includegraphics[width=\linewidth]{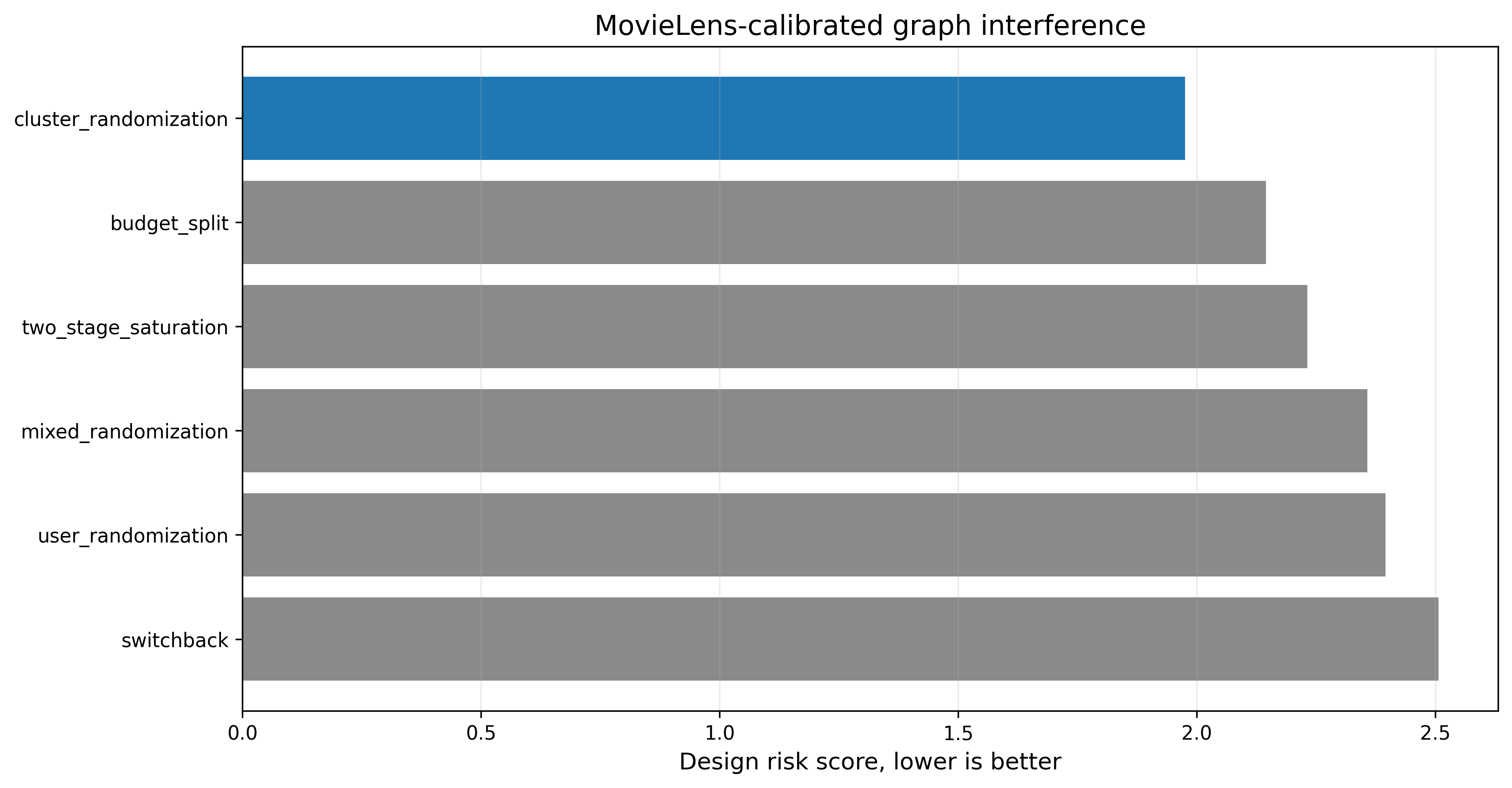}
\caption{MovieLens}
\end{subfigure}
\caption{\small Detailed robust-risk rankings for the domain cases. Panels (a)--(c) are summarized in Table~\ref{tab:main-results}; panel (d) is the graph-recommendation appendix case summarized in Table~\ref{tab:appendix-supporting-cases}.}
\label{fig:appendix-domain-risk-rankings}
\end{figure}

Figure~\ref{fig:appendix-domain-frontiers} decomposes those rankings into exposure-distance and assignment-variance tradeoffs. Here in each panel, marker size represents the design's operational-cost score, so visually larger points correspond to designs that are harder to deploy, monitor, or interpret operationally. This is the diagnostic counterpart of the risk definition in Algorithm~\ref{alg:robust-selector}. It shows why a design can lose even when it is strong on one component. Switchbacks can have low contamination but high exposure mismatch or MDE. User randomization can have low operating cost but high assignment-unit variance under clustered or repeated-exposure mechanisms. Clustered or saturation-style designs can improve exposure geometry while paying in power or implementation burden. The frontiers therefore make the selector auditable as the final recommendation is not a black-box label, but the outcome of visible component tradeoffs.

\begin{figure}[p]
\centering
\begin{subfigure}{0.48\linewidth}
\centering
\includegraphics[width=\linewidth]{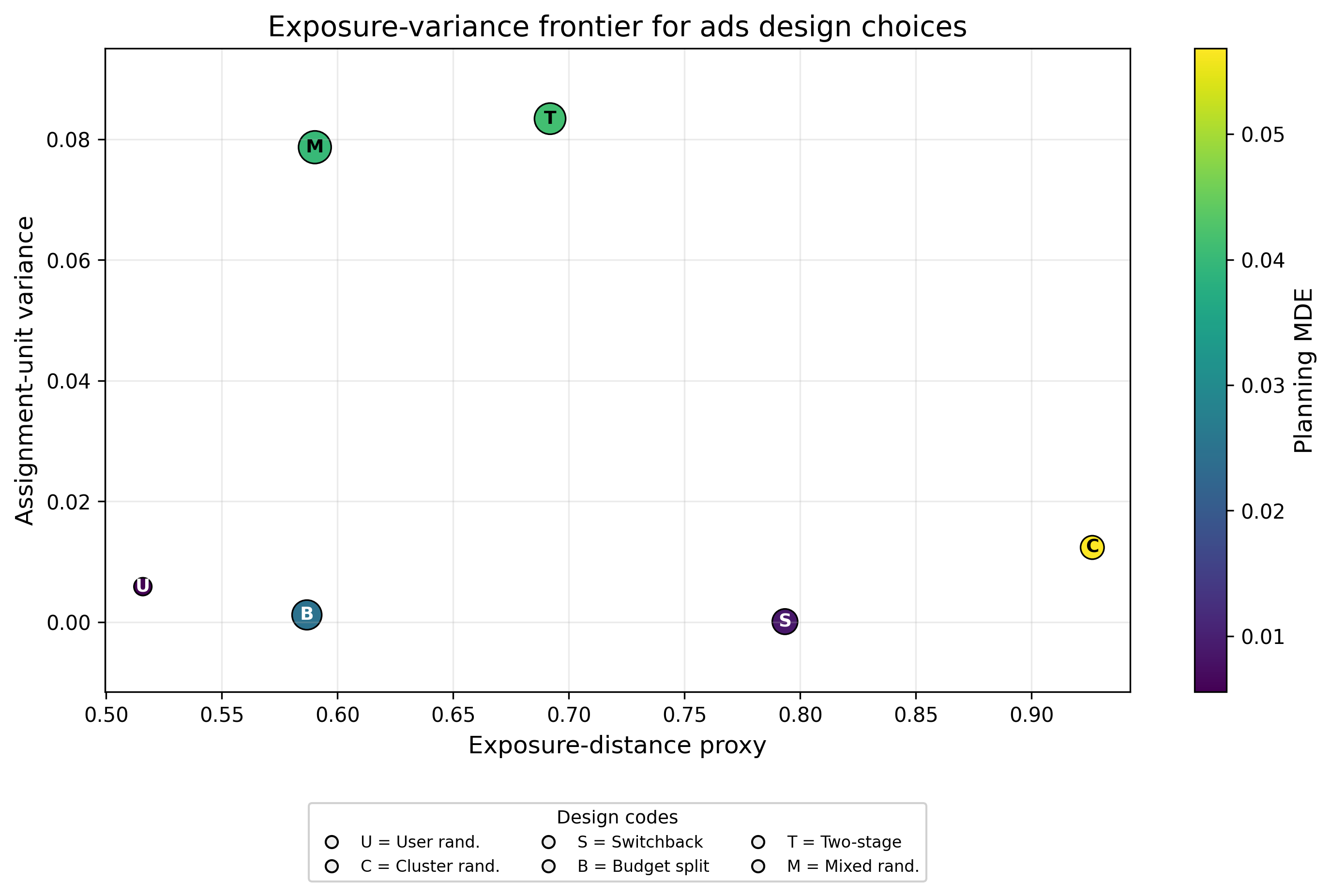}
\caption{Criteo ads}
\end{subfigure}
\hfill
\begin{subfigure}{0.48\linewidth}
\centering
\includegraphics[width=\linewidth]{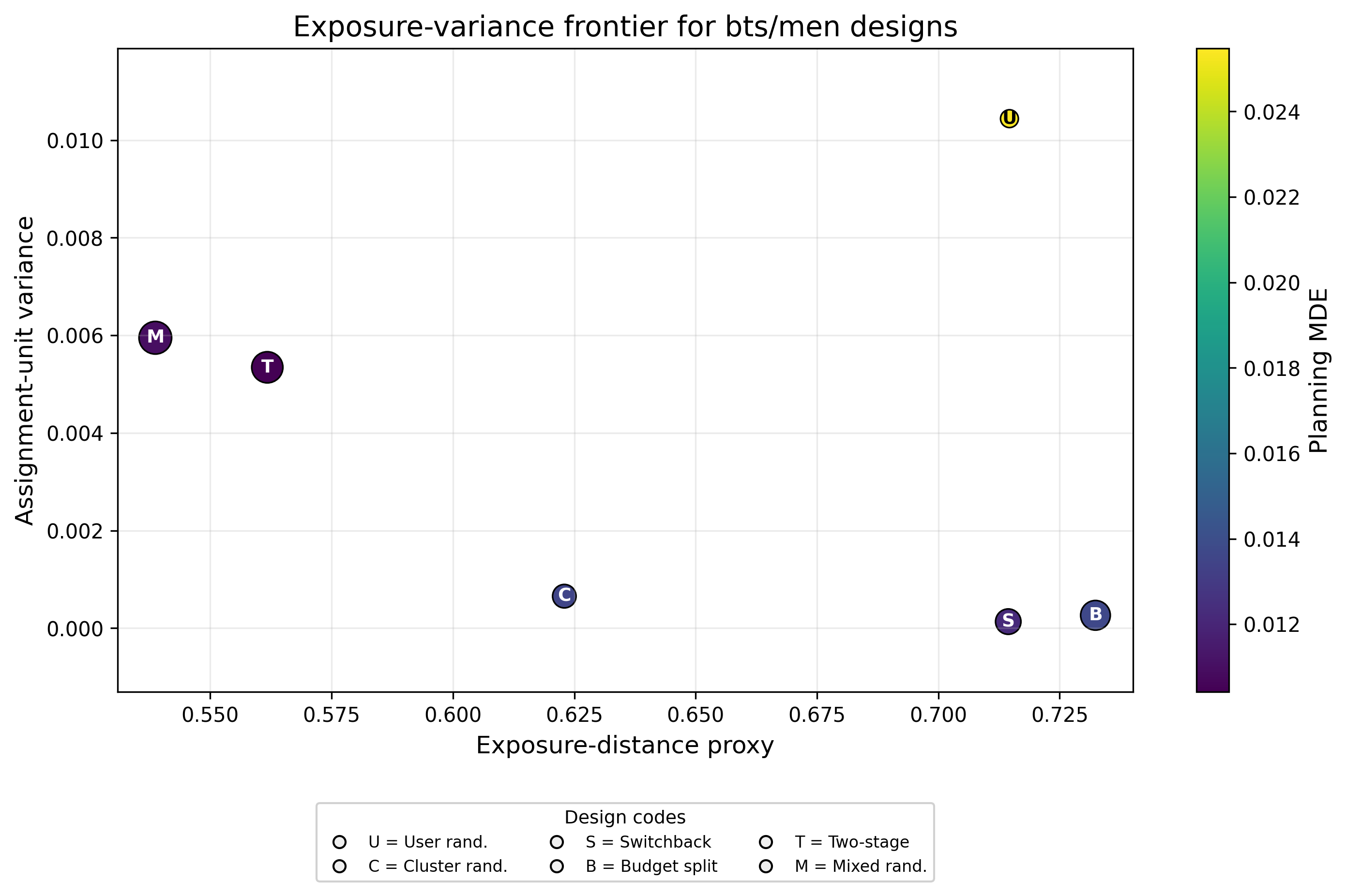}
\caption{Open Bandit \texttt{bts/men}}
\end{subfigure}

\vspace{0.6em}

\begin{subfigure}{0.48\linewidth}
\centering
\includegraphics[width=\linewidth]{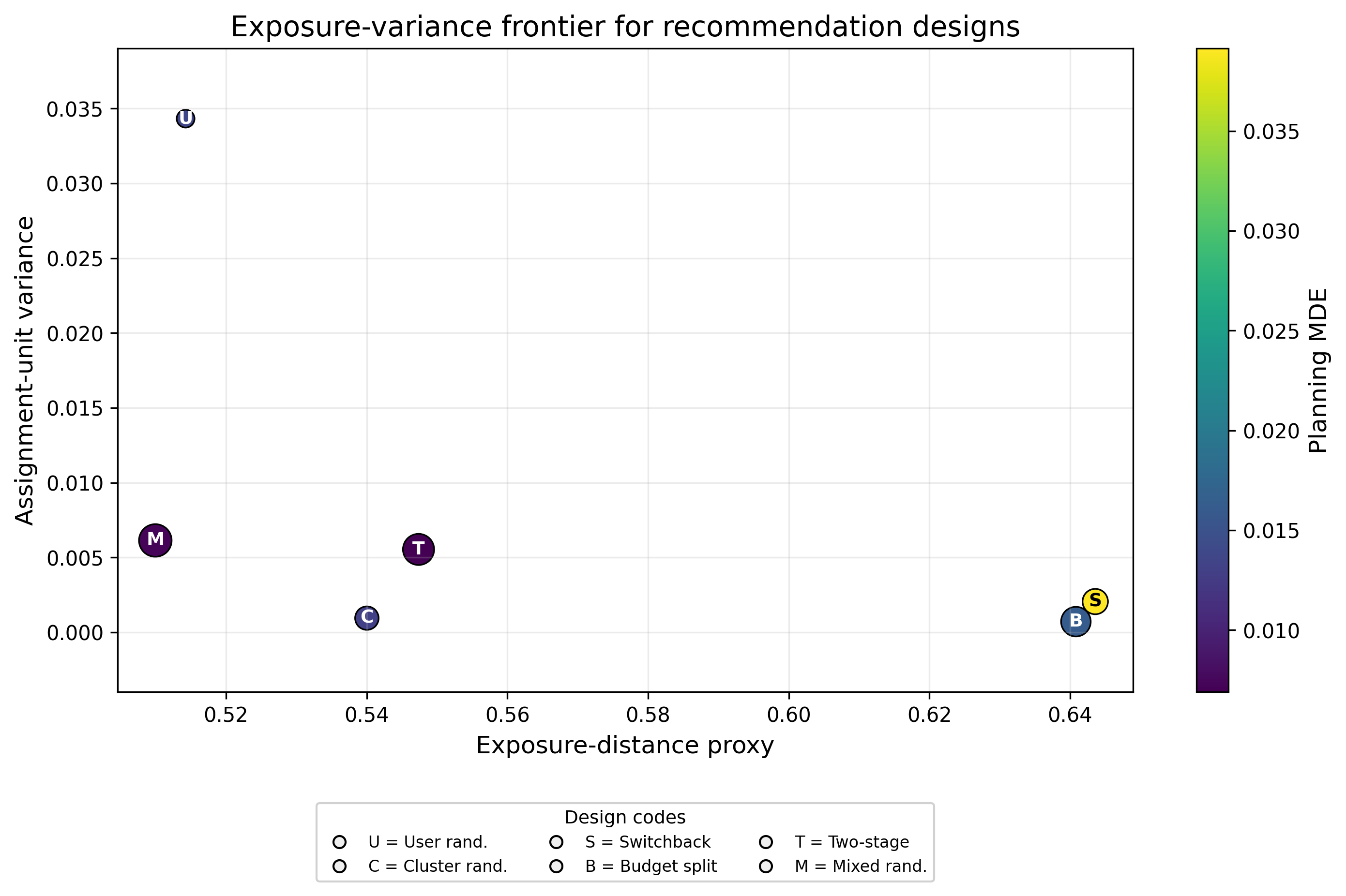}
\caption{KuaiRand}
\end{subfigure}
\hfill
\begin{subfigure}{0.48\linewidth}
\centering
\includegraphics[width=\linewidth]{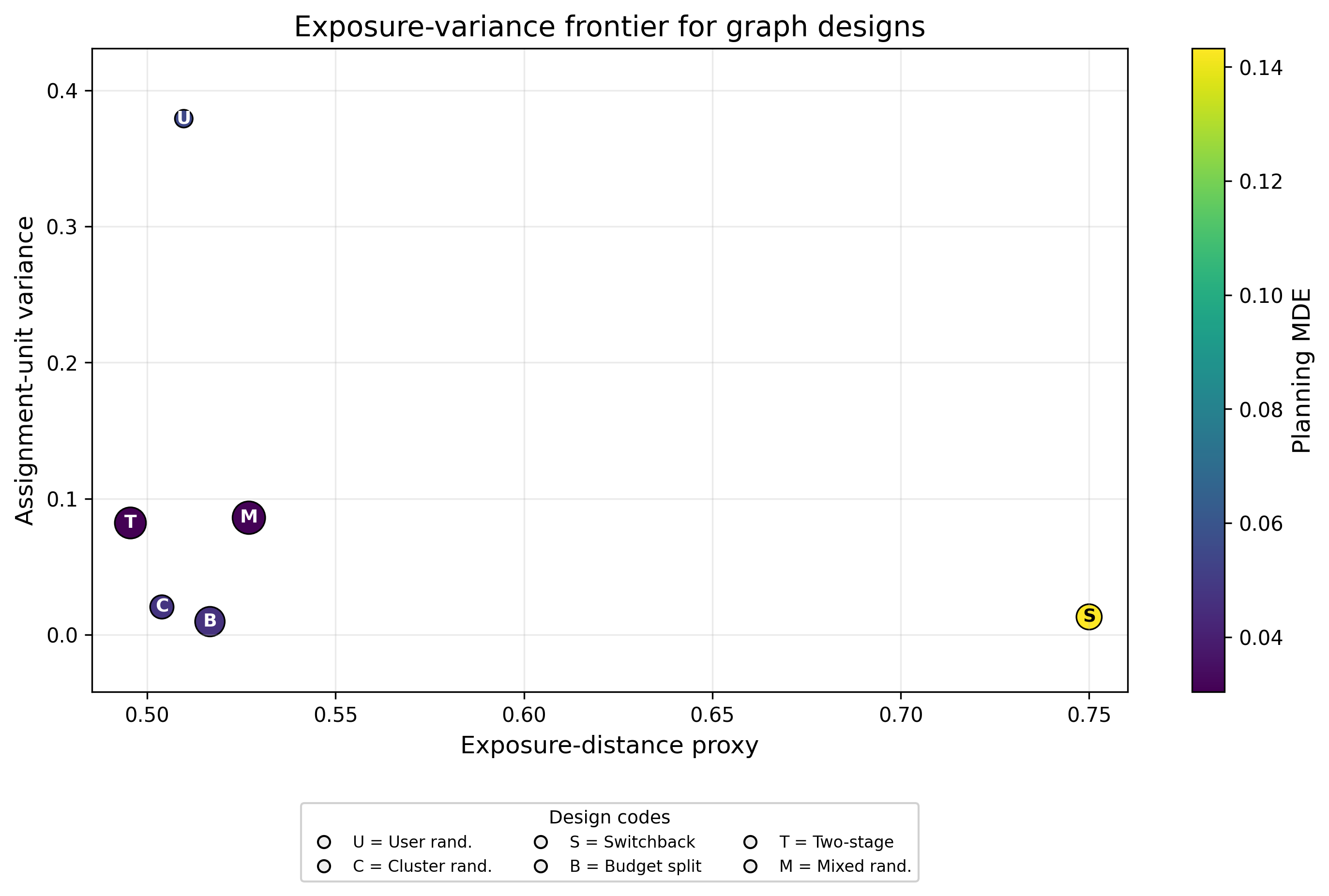}
\caption{MovieLens}
\end{subfigure}
\caption{\small Exposure--variance frontiers for the domain cases. The horizontal axis is the exposure-distance proxy and the vertical axis is assignment-unit variance. Color encodes planning MDE, while marker size is proportional to the pre-specified operational-cost score. These plots show the component tradeoffs behind the aggregate robust-risk recommendations.}
\label{fig:appendix-domain-frontiers}
\end{figure}

% The selector diagnostics in Figure~\ref{fig:appendix-selector-diagnostics} connect the empirical workflow back to Lemma~\ref{lem:no-universal-design}, Proposition~\ref{prop:regime-threshold}, and Theorem~\ref{thm:geometry-robust-selector}. Panel (a) displays a controlled regime-reversal map in which the selected design changes as the exposure mechanism moves from row-local effects to mixed spillover, clustered spillover, and carryover. Panel (b) compares a noisy planning-risk selector to a higher-replication oracle. The selected design agrees with the oracle in this controlled check, and the risk gap remains inside the theorem-guided tolerance. This is the appendix evidence that the shortlist logic is not merely descriptive; it is the finite-sample decision object predicted by the theory.

The selector diagnostics in Figure~\ref{fig:appendix-selector-diagnostics} connect the empirical workflow back to Lemma~\ref{lem:no-universal-design}, Proposition~\ref{prop:regime-threshold}, and Theorem~\ref{thm:geometry-robust-selector}. These diagnostics use the controlled synthetic platform rather than one of the public datasets, because the goal is to check selector behavior in a setting where the exposure mechanism and true launch effect are known. Panel~(a) sweeps the mechanism from weak row-local interference to mixed spillover, clustered spillover, and carryover-dominant interference. The selected design changes across this sweep, showing that the selector is not hard-coded to prefer a single randomization scheme. Panel~(b) compares the design chosen by the ordinary empirical selector, computed from a small number (45) of assignment replays, with an oracle benchmark computed from many more (260) simulation replications under the same data-generating mechanism. The oracle is not an external ground truth from production; it is a high-replication approximation to the population planning-risk minimizer in the controlled simulation. In this diagnostic, the empirical selector chooses the same design as the oracle, and the robust-risk gap remains inside the theorem-guided tolerance. This supports the interpretation of the shortlist as a finite-sample decision object rather than a purely descriptive ranking.

\begin{figure}[t]
\centering
\begin{subfigure}{0.48\linewidth}
\centering
\includegraphics[width=\linewidth]{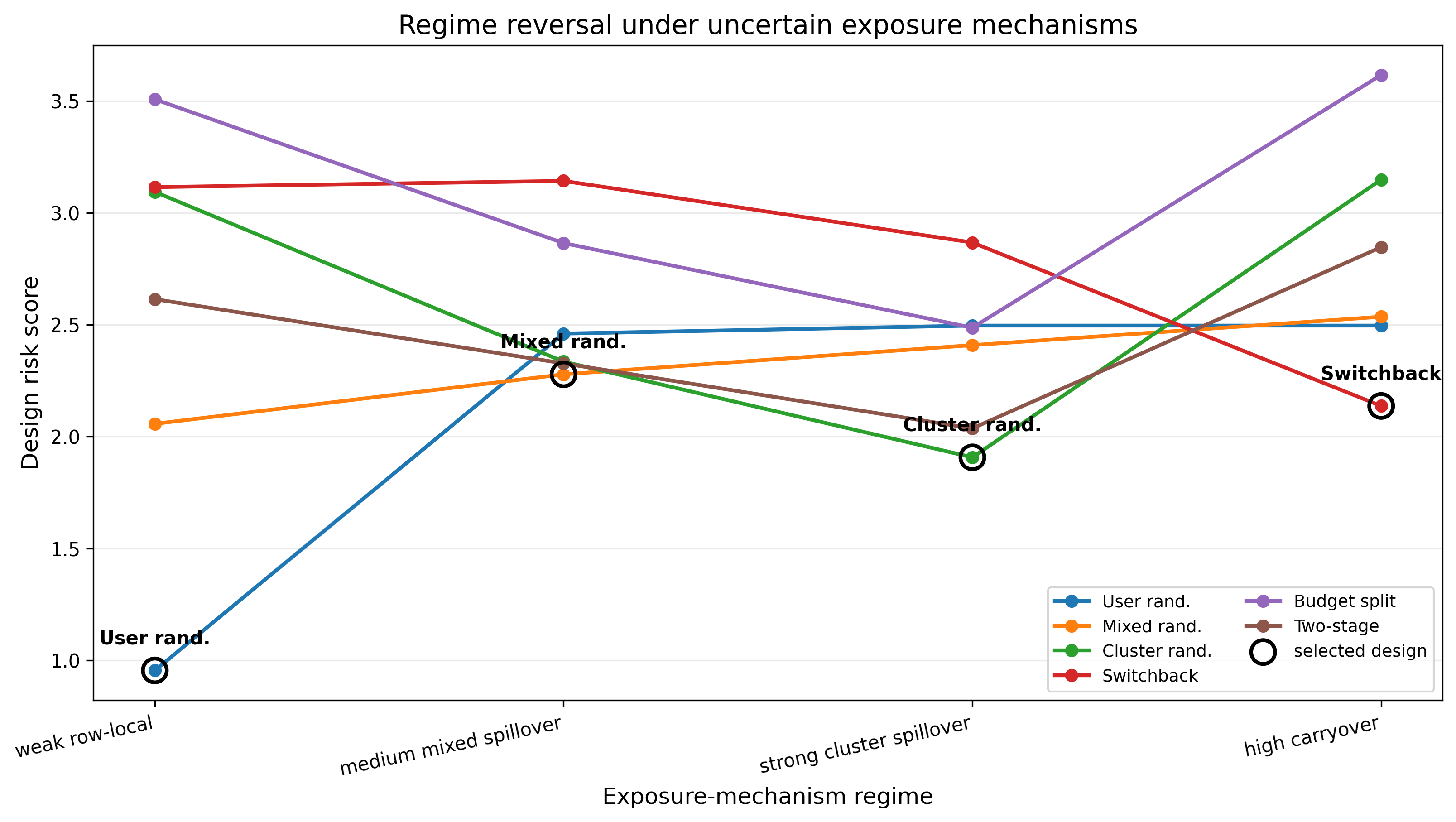}
\caption{Regime-reversal map}
\end{subfigure}
\hfill
\begin{subfigure}{0.48\linewidth}
\centering
\includegraphics[width=\linewidth]{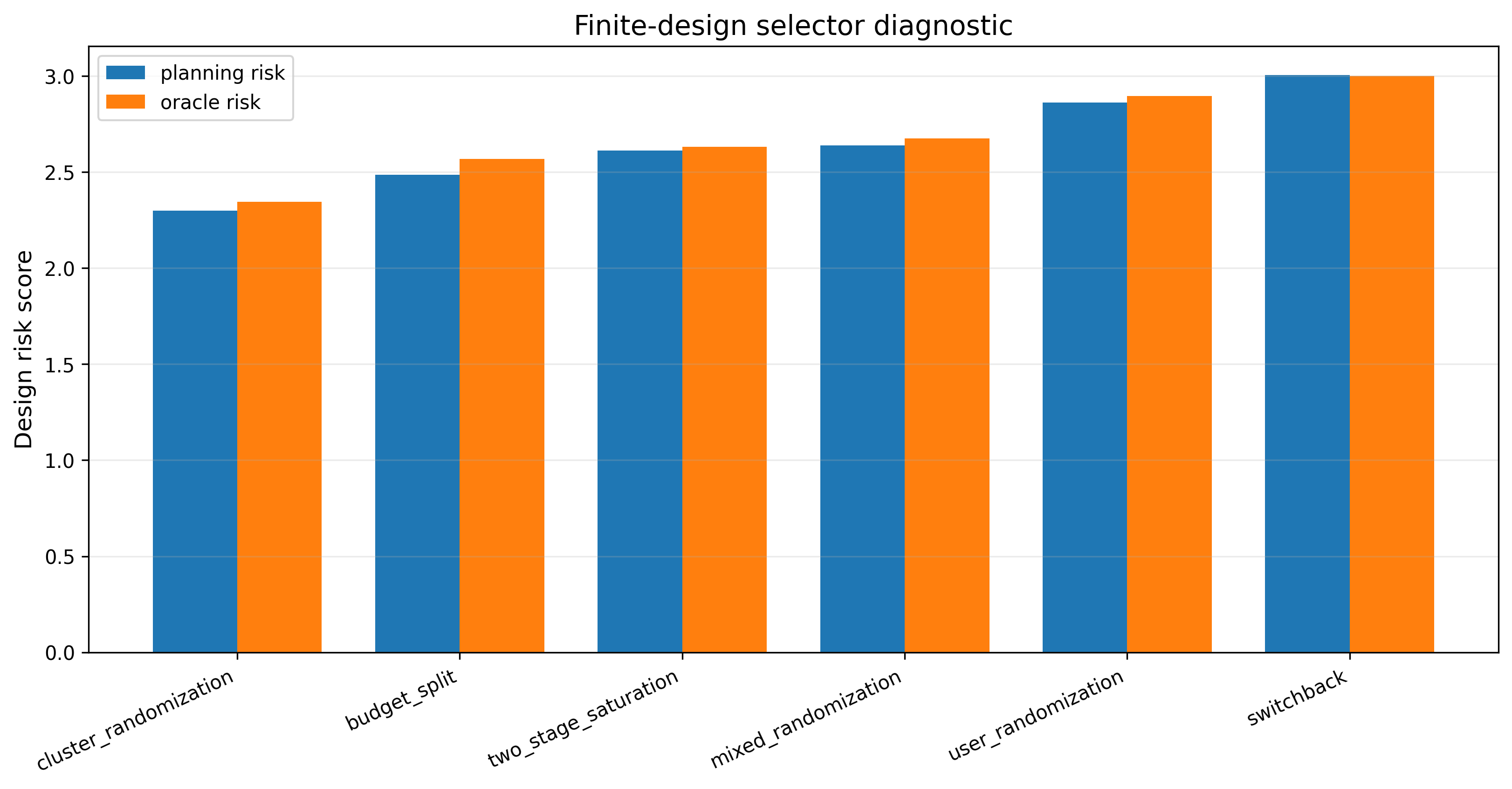}
\caption{Selector diagnostic}
\end{subfigure}
\caption{\small Appendix selector diagnostics. Panel (a) shows a controlled regime-reversal map, where the selected design changes as the dominant exposure mechanism changes. Panel (b) compares empirical and oracle planning risks; the selected design matches the oracle in the diagnostic and satisfies the excess-risk certificate from Theorem~\ref{thm:geometry-robust-selector}.}
\label{fig:appendix-selector-diagnostics}
\end{figure}

Figure~\ref{fig:appendix-theory-checks} isolates the mathematical components of the framework. All four panels in Figure~\ref{fig:appendix-theory-checks} are controlled synthetic diagnostics from Notebook~07 (in the codebase) rather than public-dataset results. Panels~(a) and~(b) study the exposure-geometry term in Theorems~\ref{thm:geometry-bias} and~\ref{thm:geometry-necessity}. In panel~(a), the simulated exposure feature is one-dimensional. The experimental exposure distribution is a beta-distributed baseline, and the launch exposure distribution is obtained by shifting that exposure by a controlled amount. For an \(L\)-Lipschitz exposure-response function \(r\), the theorem gives the upper bound
\[
\left|
\int r\,dP_d^\theta-\int r\,dP_\star^\theta
\right|
\le
L\,W_1(P_d^\theta,P_\star^\theta).
\]
The horizontal axis plots the right-hand side \(L\,W_1(P_d^\theta,P_\star^\theta)\), and the vertical axis plots the observed exposure-response bias. The scenarios vary the exposure shift and the Lipschitz constant. \textit{The points lie on or below the equality line, confirming the transport-based bias bound in the constructed cases.}

Panel~(b) checks the minimax statement behind the same geometry term. Here the constructed measures are point-mass-like exposure distributions separated by a shift \(\delta\), and the response function is \(r(x)=Lx\). This response is \(L\)-Lipschitz and attains the Wasserstein dual in this ordered one-dimensional construction, so
\[
\left|
\int r\,dP-\int r\,dQ
\right|
=
L\,W_1(P,Q).
\]
The plotted quantity is the ratio of the attained gap to the minimax penalty \(L\,W_1(P,Q)\). A ratio of one means that the upper bound is tight. The blue, orange, and green curves correspond to different Lipschitz constants, but they overlap because all three attain the minimax equality exactly in this construction. This is why the visible points appear green even though the legend contains three values of \(L\). The minimax interpretation matters because it says the Wasserstein penalty is not simply a conservative heuristic. If the analyst assumes only Lipschitz exposure response, then an adversarial response function in that class can make the bias as large as the penalty.

Panel~(c) studies Proposition~\ref{prop:catalog-approximation}. The code defines a smooth one-dimensional robust-risk surface over a continuous design parameter and then restricts attention to finite catalogs of increasing size. The plotted catalog gap is
\[
\min_{d\in\D_{\mathrm{cat}}}\mathcal R_{\mathrm{rob}}(d)
-
\inf_{d\in\mathfrak D}\mathcal R_{\mathrm{rob}}(d),
\]
while the upper-bound curve is the Lipschitz-net bound
\[
L_{\mathfrak D}\eta.
\]
As the catalog becomes finer, the net radius \(\eta\) shrinks and the observed approximation gap remains below the bound. This diagnostic supports the use of finite implementable design catalogs rather than an unrestricted search over all randomization schemes.

Panel~(d) studies the assignment-unit MDE calculation used in the planning risk. For each design, the code fixes a design-specific number of effective assignment units per week and a design-specific assignment-unit standard deviation. It then computes
\[
\MDE_d(T)
=
(z_{1-\alpha/2}+z_{1-\beta})
\sqrt{\frac{2\sigma_d^2}{N_d(T)}}.
\]
The panel shows that detectable effects shrink with the effective number of assignment units \(N_d(T)\) and with lower assignment-unit variance \(\sigma_d^2\). The relevant sample size is therefore not the number of logged rows, but the number of independent units produced by the design. Together, these four diagnostics validate the components that the robust selector aggregates in the real-data experiments.

\begin{figure}[p]
\centering
\begin{subfigure}{0.48\linewidth}
\centering
\includegraphics[width=\linewidth]{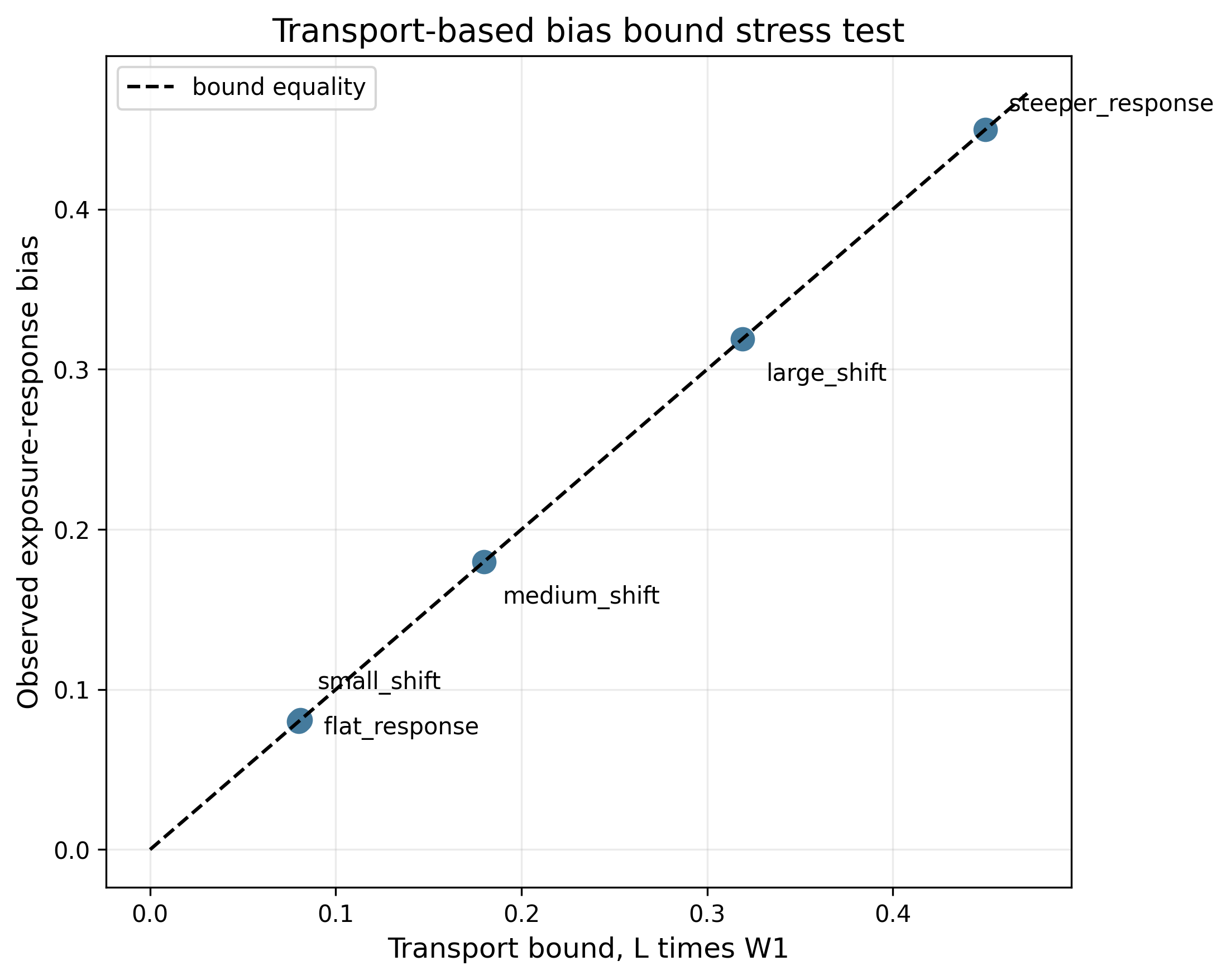}
\caption{Transport bias}
\end{subfigure}
\hfill
\begin{subfigure}{0.48\linewidth}
\centering
\includegraphics[width=\linewidth]{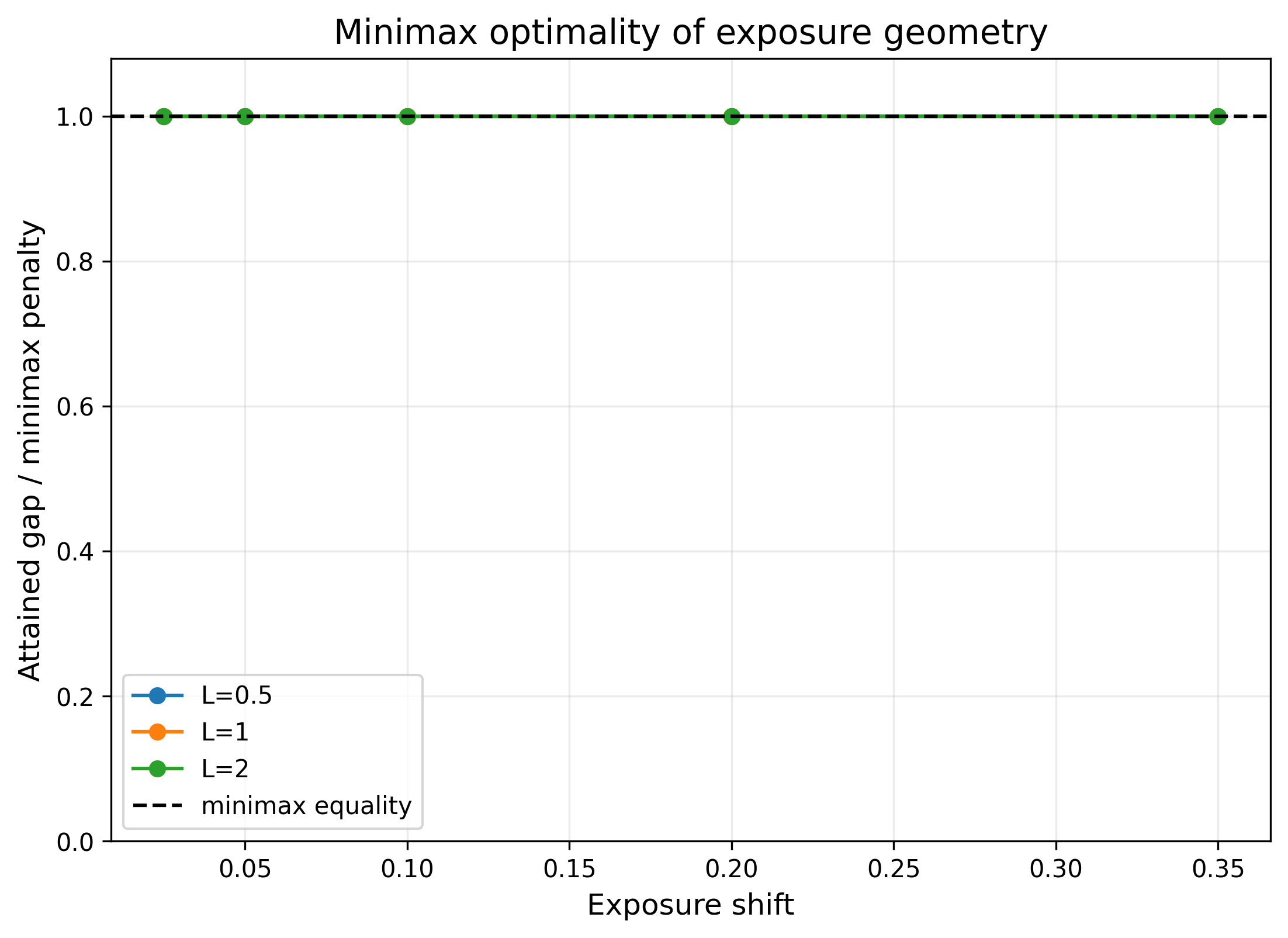}
\caption{Minimax tightness}
\end{subfigure}

\vspace{0.6em}

\begin{subfigure}{0.48\linewidth}
\centering
\includegraphics[width=\linewidth]{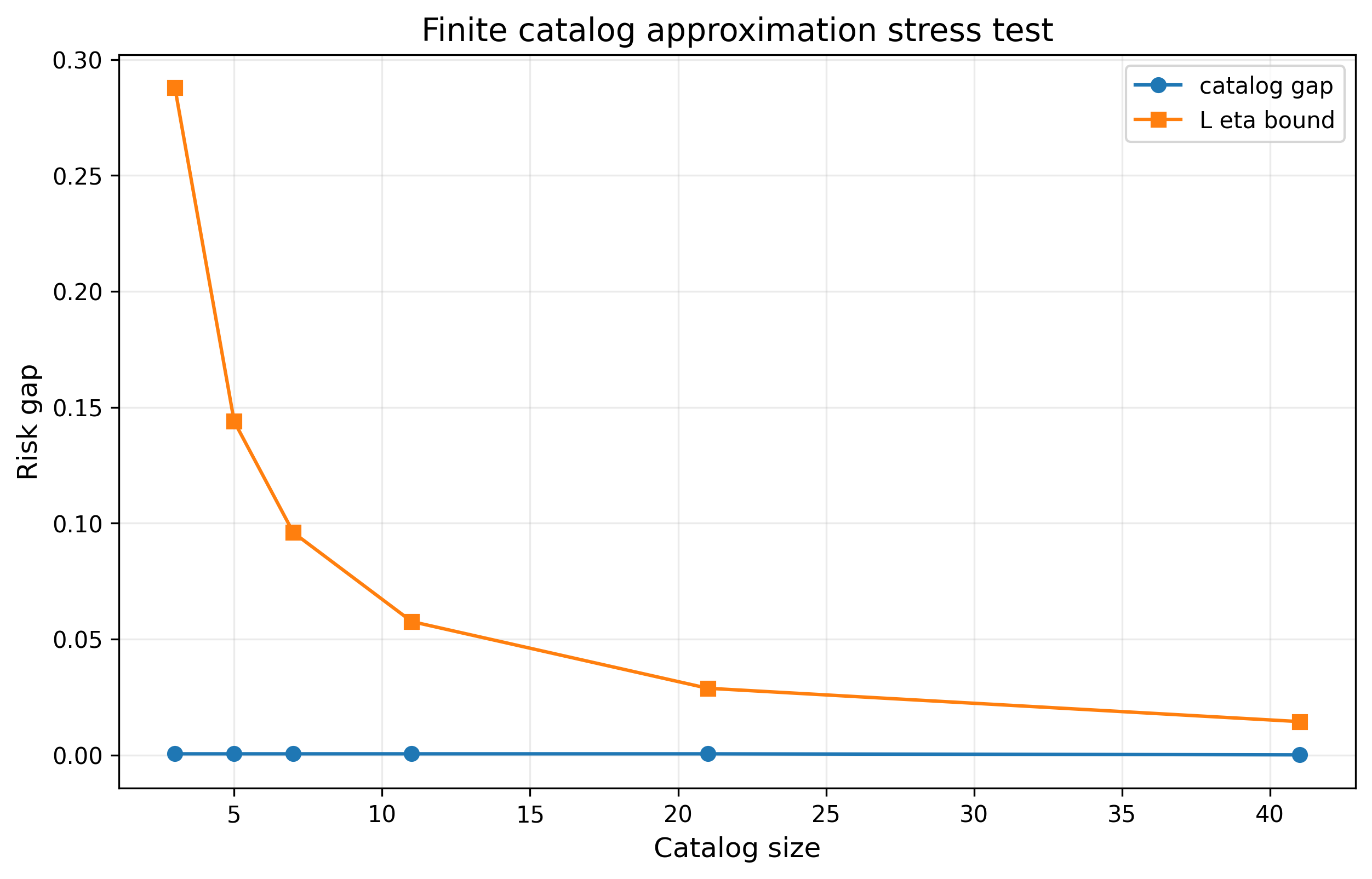}
\caption{Catalog approximation}
\end{subfigure}
\hfill
\begin{subfigure}{0.48\linewidth}
\centering
\includegraphics[width=\linewidth]{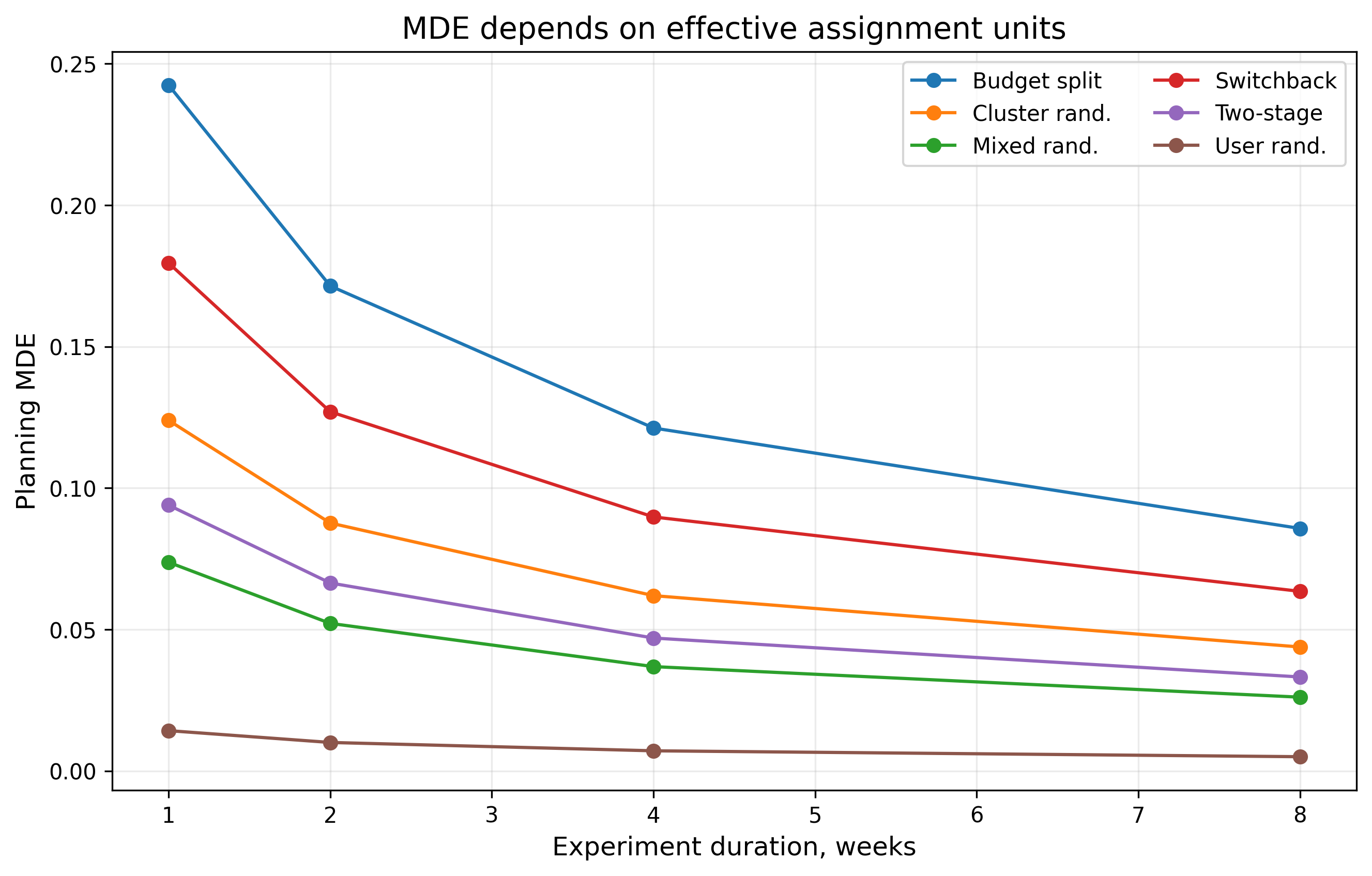}
\caption{Assignment-unit MDE}
\end{subfigure}
\caption{\small Theory stress tests. Panel~(a) checks the transport bias bound by plotting observed exposure-response bias against \(L\,W_1(P_d^\theta,P_\star^\theta)\). Panel~(b) checks minimax tightness by plotting the ratio of the attained Lipschitz gap to the minimax penalty; the curves for different \(L\) overlap at one because the constructed response attains equality. Panel~(c) compares finite-catalog approximation error with the Lipschitz-net upper bound. Panel~(d) shows assignment-unit MDE as a function of duration and effective assignment units.}
\label{fig:appendix-theory-checks}
\end{figure}

\section{Foundational Results and Proofs}

\subsection{No uniformly best default design}

The main text focuses on computable geometry-aware results. The following foundational lemma explains why a single default design is not defensible across all interference regimes unless it dominates every competitor on every relevant risk component.

\begin{lemma}[No uniformly best design without componentwise dominance]
\label{lem:no-universal-design}
Let each feasible design \(d\in\D_{\mathrm{feas}}\) have a nonnegative risk-vector
\[
q(d;\theta)
=
\big(\widetilde G_d(\theta),\widetilde V_d(\theta),\widetilde M_d(T;\theta),\widetilde C_d(\theta),\widetilde O_d,\widetilde E_d(\theta)\big)
\]
for every mechanism \(\theta\in\U\). Suppose a design \(d_0\) minimizes \(w^\top q(d;\theta)\) over \(\D_{\mathrm{feas}}\) for every \(\theta\in\U\) and every nonnegative weight vector \(w\). Then \(d_0\) componentwise dominates every other design, meaning \(q_k(d_0;\theta)\le q_k(d;\theta)\) for every component \(k\), every design \(d\), and every mechanism \(\theta\).
\end{lemma}

\begin{proof}
Fix any mechanism \(\theta\), design \(d\), and component \(k\). If \(q_k(d_0;\theta)>q_k(d;\theta)\), choose the nonnegative weight vector \(w\) with weight one on component \(k\) and zero on every other component. Then \(w^\top q(d;\theta)<w^\top q(d_0;\theta)\), contradicting the assumed optimality of \(d_0\) for all nonnegative weights. Therefore \(q_k(d_0;\theta)\le q_k(d;\theta)\) for all \(k,d,\theta\).
\end{proof}

\subsection{Proof of Theorem~\ref{thm:geometry-bias}}

% \begin{proof}
% For a fixed mechanism \(\theta\), the representation assumptions give
% \[
% \tau_d(\theta)=\int r_\theta(\phi)\,dP_d^\theta(\phi),
% \qquad
% \left|
% \tau^\star(\theta)
% -
% \int r_\theta(\phi)\,dP_\star^\theta(\phi)
% \right|
% \le E_d(\theta).
% \]
% Therefore
% \[
% |B_d(\theta)|
% \le
% \left|
% \int r_\theta\,dP_d^\theta-\int r_\theta\,dP_\star^\theta
% \right|
% +E_d(\theta).
% \]
% By the Kantorovich--Rubinstein dual representation of Wasserstein-1 distance, for any \(L_\theta\)-Lipschitz function \(r_\theta\),
% \[
% \left|
% \int r_\theta\,dP_d^\theta-\int r_\theta\,dP_\star^\theta
% \right|
% \le
% L_\theta \Wass(P_d^\theta,P_\star^\theta).
% \]
% Since \(G_d(\theta)=\Wass(P_d^\theta,P_\star^\theta)\), the result follows.
% \end{proof}

\begin{proof}
The proof has two steps. First, we separate the bias of the design estimand from any residual mismatch between the launch estimand and the exposure-response representation. Second, we control the design-estimand difference by Wasserstein distance.

By definition,
\[
B_d(\theta)
=
\tau_d(\theta)-\tau^\star(\theta).
\]
Add and subtract the launch exposure-response functional
\[
\int r_\theta(\phi)\,dP_\star^\theta(\phi)
\]
to obtain
\[
B_d(\theta)
=
\left[
\tau_d(\theta)
-
\int r_\theta(\phi)\,dP_\star^\theta(\phi)
\right]
+
\left[
\int r_\theta(\phi)\,dP_\star^\theta(\phi)
-
\tau^\star(\theta)
\right].
\]
Using the assumed representation
\[
\tau_d(\theta)
=
\int r_\theta(\phi)\,dP_d^\theta(\phi),
\]
this becomes
\[
B_d(\theta)
=
\left[
\int r_\theta(\phi)\,dP_d^\theta(\phi)
-
\int r_\theta(\phi)\,dP_\star^\theta(\phi)
\right]
+
\left[
\int r_\theta(\phi)\,dP_\star^\theta(\phi)
-
\tau^\star(\theta)
\right].
\]
Taking absolute values and applying the triangle inequality gives
\[
|B_d(\theta)|
\le
\left|
\int r_\theta(\phi)\,dP_d^\theta(\phi)
-
\int r_\theta(\phi)\,dP_\star^\theta(\phi)
\right|
+
\left|
\int r_\theta(\phi)\,dP_\star^\theta(\phi)
-
\tau^\star(\theta)
\right|.
\]
The second term is bounded by \(E_d(\theta)\) by assumption. For the first term, use the Kantorovich--Rubinstein dual representation of Wasserstein-1 distance. Since \(r_\theta\) is \(L_\theta\)-Lipschitz,
\[
\left|
\int r_\theta(\phi)\,dP_d^\theta(\phi)
-
\int r_\theta(\phi)\,dP_\star^\theta(\phi)
\right|
\le
L_\theta W_1(P_d^\theta,P_\star^\theta).
\]
By definition,
\[
G_d(\theta)
=
W_1(P_d^\theta,P_\star^\theta).
\]
Therefore,
\[
|B_d(\theta)|
\le
L_\theta G_d(\theta)+E_d(\theta).
\]

It remains only to connect this estimand bound to the estimator. The theorem assumes that
\[
\mathbb E_\theta[\widehat \tau_d]=\tau_d(\theta),
\]
so the estimator has no additional design-specific bias relative to the design estimand. Hence the bias controlled above is exactly the mismatch between the expected experimental estimand and the launch estimand. This proves the claim.
\end{proof}

\subsection{Proof of Theorem~\ref{thm:geometry-necessity}}

% \begin{proof}
% The Kantorovich--Rubinstein dual representation of Wasserstein-1 distance on a Polish metric space with finite first moments gives
% \[
% \Wass(P,Q)
% =
% \sup_{r\in\mathcal F_1}
% \left|
% \int r\,dP-\int r\,dQ
% \right|,
% \]
% where \(\mathcal F_1\) is the set of all 1-Lipschitz measurable functions. Scaling the function class by \(L\) gives
% \[
% \sup_{r\in\mathcal F_L}
% \left|
% \int r\,dP-\int r\,dQ
% \right|
% =
% L\Wass(P,Q).
% \]
% For the residual estimand-mismatch statement, the upper bound follows from the triangle inequality:
% \[
% \left|
% \int r_\theta\,dP_d^\theta
% -
% \int r_\theta\,dP_\star^\theta
% -
% \Delta_E
% \right|
% \le
% L_\theta G_d(\theta)+E_d(\theta).
% \]
% For the reverse inequality, fix any \(\eta>0\). By the first part, there exists \(r_\eta\in\mathcal F_{L_\theta}\) such that
% \[
% \left|
% \int r_\eta\,dP_d^\theta
% -
% \int r_\eta\,dP_\star^\theta
% \right|
% \ge
% L_\theta G_d(\theta)-\eta .
% \]
% Choose \(\Delta_E\) with magnitude \(E_d(\theta)\) and sign opposite to the displayed integral difference. Then
% \[
% \left|
% \int r_\eta\,dP_d^\theta
% -
% \int r_\eta\,dP_\star^\theta
% -
% \Delta_E
% \right|
% \ge
% L_\theta G_d(\theta)+E_d(\theta)-\eta .
% \]
% Letting \(\eta\downarrow0\) proves equality.
% \end{proof}
\begin{proof}
The proof follows from the Kantorovich--Rubinstein dual representation of the Wasserstein-1 distance, plus a simple scaling argument and an adversarial residual-mismatch argument.

Because \(P\) and \(Q\) have finite first moments on the metric space \((\M,\rho)\), the Wasserstein-1 distance is finite and admits the dual representation
\[
W_1(P,Q)
=
\sup_{f\in\mathcal F_1}
\left\{
\int f\,dP-\int f\,dQ
\right\},
\]
where
\[
\mathcal F_1
=
\{f:\M\to\R:\ |f(x)-f(y)|\le \rho(x,y)\ \text{for all }x,y\in\M\}
\]
is the class of 1-Lipschitz functions. Since \(f\in\mathcal F_1\) implies \(-f\in\mathcal F_1\), the same dual representation can be written with an absolute value:
\[
W_1(P,Q)
=
\sup_{f\in\mathcal F_1}
\left|
\int f\,dP-\int f\,dQ
\right|.
\]

Now consider the \(L\)-Lipschitz class
\[
\mathcal F_L
=
\{r:\M\to\R:\ |r(x)-r(y)|\le L\rho(x,y)\ \text{for all }x,y\in\M\}.
\]
If \(r\in\mathcal F_L\), then \(f=r/L\in\mathcal F_1\). Conversely, if \(f\in\mathcal F_1\), then \(r=Lf\in\mathcal F_L\). Therefore,
\[
\sup_{r\in\mathcal F_L}
\left|
\int r\,dP-\int r\,dQ
\right|
=
\sup_{f\in\mathcal F_1}
\left|
\int Lf\,dP-\int Lf\,dQ
\right|.
\]
Pulling out the constant \(L\),
\[
\sup_{r\in\mathcal F_L}
\left|
\int r\,dP-\int r\,dQ
\right|
=
L
\sup_{f\in\mathcal F_1}
\left|
\int f\,dP-\int f\,dQ
\right|.
\]
By the Kantorovich--Rubinstein duality above, this equals
\[
L W_1(P,Q).
\]
This proves the first claim.

For the second claim, set
\[
D(r_\theta)
=
\int r_\theta\,dP_d^\theta
-
\int r_\theta\,dP_\star^\theta .
\]
The residual term \(\Delta_E\) is allowed to range over all values satisfying
\[
|\Delta_E|\le E_d(\theta).
\]
For any fixed \(r_\theta\), we have
\[
\sup_{|\Delta_E|\le E_d(\theta)}
\left|
D(r_\theta)-\Delta_E
\right|
=
|D(r_\theta)|+E_d(\theta).
\]
Indeed, if \(D(r_\theta)\ge 0\), choose \(\Delta_E=-E_d(\theta)\); if \(D(r_\theta)<0\), choose \(\Delta_E=E_d(\theta)\). In either case the residual mismatch is chosen with the sign that maximizes the absolute discrepancy.

Therefore,
\[
\sup_{\substack{r_\theta\in\mathcal F_{L_\theta}\\ |\Delta_E|\le E_d(\theta)}}
\left|
\int r_\theta\,dP_d^\theta
-
\int r_\theta\,dP_\star^\theta
-
\Delta_E
\right|
=
\sup_{r_\theta\in\mathcal F_{L_\theta}}
|D(r_\theta)|
+
E_d(\theta).
\]
Applying the first part of the theorem with
\[
P=P_d^\theta,\qquad
Q=P_\star^\theta,\qquad
L=L_\theta,
\]
gives
\[
\sup_{r_\theta\in\mathcal F_{L_\theta}}
|D(r_\theta)|
=
L_\theta W_1(P_d^\theta,P_\star^\theta).
\]
By definition,
\[
G_d(\theta)
=
W_1(P_d^\theta,P_\star^\theta).
\]
Hence,
\[
\sup_{\substack{r_\theta\in\mathcal F_{L_\theta}\\ |\Delta_E|\le E_d(\theta)}}
\left|
\int r_\theta\,dP_d^\theta
-
\int r_\theta\,dP_\star^\theta
-
\Delta_E
\right|
=
L_\theta G_d(\theta)+E_d(\theta).
\]

This shows that the Wasserstein exposure-distance penalty is not merely sufficient. Over the class of \(L_\theta\)-Lipschitz exposure-response functions, it is also worst-case sharp.
\end{proof}

\subsection{Proof of Proposition~\ref{prop:catalog-approximation}}

\begin{proof}
The proof has two steps. First, the finite catalog must contain a design close to any design in the larger feasible class. Second, Lipschitz continuity of the robust risk converts closeness in design space into closeness in risk.

Recall that
\[
\mathcal R_{\mathrm{rob}}(d)
=
\sup_{\theta\in\U}\mathcal R_T^{\mathrm{geo}}(d;\theta).
\]
By assumption, \(\mathcal R_{\mathrm{rob}}\) is \(L_{\mathfrak D}\)-Lipschitz on \((\mathfrak D,\rho_{\mathfrak D})\). Thus, for any two designs \(d,d'\in\mathfrak D\),
\[
\left|
\mathcal R_{\mathrm{rob}}(d)
-
\mathcal R_{\mathrm{rob}}(d')
\right|
\le
L_{\mathfrak D}\rho_{\mathfrak D}(d,d').
\]
This condition is natural here. For example, it holds whenever each component of the normalized planning risk changes Lipschitzly with the design and the constants are uniform over \(\theta\in\U\). Since the robust risk is a supremum over mechanisms, a uniform Lipschitz bound is preserved under the supremum.

Let
\[
R^\star
=
\inf_{d\in\mathfrak D}\mathcal R_{\mathrm{rob}}(d).
\]
The infimum need not be attained, so fix an arbitrary \(\delta>0\). By the definition of infimum, there exists a design \(d^\delta\in\mathfrak D\) such that
\[
\mathcal R_{\mathrm{rob}}(d^\delta)
\le
R^\star+\delta.
\]
Because \(\D_{\mathrm{cat}}\) is an \(\eta\)-net of \(\mathfrak D\), there exists a catalog design \(d_{\mathrm{cat}}^\delta\in\D_{\mathrm{cat}}\) satisfying
\[
\rho_{\mathfrak D}(d^\delta,d_{\mathrm{cat}}^\delta)
\le
\eta.
\]
Applying Lipschitz continuity gives
\[
\mathcal R_{\mathrm{rob}}(d_{\mathrm{cat}}^\delta)
\le
\mathcal R_{\mathrm{rob}}(d^\delta)
+
L_{\mathfrak D}
\rho_{\mathfrak D}(d^\delta,d_{\mathrm{cat}}^\delta).
\]
Using the \(\eta\)-net property,
\[
\mathcal R_{\mathrm{rob}}(d_{\mathrm{cat}}^\delta)
\le
\mathcal R_{\mathrm{rob}}(d^\delta)
+
L_{\mathfrak D}\eta.
\]
Using the near-optimality of \(d^\delta\),
\[
\mathcal R_{\mathrm{rob}}(d_{\mathrm{cat}}^\delta)
\le
R^\star+\delta+L_{\mathfrak D}\eta.
\]
Since the best catalog design has risk no larger than any particular catalog design,
\[
\min_{d\in\D_{\mathrm{cat}}}
\mathcal R_{\mathrm{rob}}(d)
\le
\mathcal R_{\mathrm{rob}}(d_{\mathrm{cat}}^\delta)
\le
R^\star+\delta+L_{\mathfrak D}\eta.
\]
Finally, \(\delta>0\) was arbitrary. Letting \(\delta\downarrow0\) yields
\[
\min_{d\in\D_{\mathrm{cat}}}
\mathcal R_{\mathrm{rob}}(d)
\le
\inf_{d\in\mathfrak D}
\mathcal R_{\mathrm{rob}}(d)
+
L_{\mathfrak D}\eta.
\]
This proves the claim.
\end{proof}

\subsection{Proof of Proposition~\ref{prop:regime-threshold}}

\begin{proof}
Under the stated surrogate risk expression, design \(d_2\) has smaller surrogate risk than \(d_1\) exactly when
\begin{align*}
&w_g\gamma g_2+w_v\bar V_2+w_m\bar M_2(T)+w_c\bar C_2+w_o\bar O_2+w_e\bar E_2\\
&\hspace{2em}<
w_g\gamma g_1+w_v\bar V_1+w_m\bar M_1(T)+w_c\bar C_1+w_o\bar O_1+w_e\bar E_1.
\end{align*}
Moving all non-\(\gamma\) terms to the right side and collecting the \(\gamma\) terms gives
\[
w_g\gamma(g_1-g_2)
>
w_v(\bar V_2-\bar V_1)+w_m(\bar M_2(T)-\bar M_1(T))+w_c(\bar C_2-\bar C_1)+w_o(\bar O_2-\bar O_1)+w_e(\bar E_2-\bar E_1).
\]
Since \(g_2<g_1\) and the stated condition requires \(w_g(g_1-g_2)>0\), division yields the stated threshold.
\end{proof}

\subsection{Proof of Theorem~\ref{thm:geometry-robust-selector}}

\begin{proof}
First bound the uniform error in the geometry-aware risk surface. By the componentwise error bounds and the nonnegative planning weights,
\begin{align*}
&\left|
\widehat{\mathcal R}_T^{\mathrm{geo}}(d;\theta)
-\mathcal R_T^{\mathrm{geo}}(d;\theta)
\right|\\
&\quad\le
w_g|\widetilde G_d(\theta)-\bar G_d(\theta)|
+w_v|\widetilde V_d(\theta)-\bar V_d(\theta)|
+w_m|\widetilde M_d(T;\theta)-\bar M_d(T;\theta)|\\
&\qquad
+w_c|\widetilde C_d(\theta)-\bar C_d(\theta)|
+w_o|\widetilde O_d-\bar O_d|
+w_e|\widetilde E_d(\theta)-\bar E_d(\theta)|\\
&\quad\le
\epsilon_T
\end{align*}
uniformly over \(d\in\D_{\mathrm{cat}}\) and \(\theta\in\U\).

Let \(Q(d)=\sup_{\theta\in\U}\mathcal R_T^{\mathrm{geo}}(d;\theta)\) and \(\widehat Q(d)=\sup_{\theta\in\U}\widehat{\mathcal R}_T^{\mathrm{geo}}(d;\theta)\). The uniform risk bound implies \(|\widehat Q(d)-Q(d)|\le\epsilon_T\) for all \(d\). Hence
\[
Q(\widehat d_T)
\le
\widehat Q(\widehat d_T)+\epsilon_T
\le
\widehat Q(d_T^\star)+\epsilon_T
\le
Q(d_T^\star)+2\epsilon_T,
\]
which is the desired excess-risk bound.

For exact recovery, suppose \(\D_{\mathrm{cat}}\setminus\D^\star\) is nonempty and \(\Delta_{\mathrm{sep}}>2\epsilon_T\). If \(\widehat d_T\notin\D^\star\), then
\[
Q(\widehat d_T)-Q^\star\ge \Delta_{\mathrm{sep}}>2\epsilon_T,
\]
contradicting the excess-risk bound. Hence \(\widehat d_T\in\D^\star\). If \(\D_{\mathrm{cat}}=\D^\star\), the claim is immediate.

For the shortlist guarantee, let \(d\in\widehat{\mathcal S}_\kappa\). Since \(\widehat d_T\) minimizes \(\widehat Q\),
\[
Q(d)
\le
\widehat Q(d)+\epsilon_T
\le
\widehat Q(\widehat d_T)+\kappa+\epsilon_T
\le
\widehat Q(d_T^\star)+\kappa+\epsilon_T
\le
Q^\star+\kappa+2\epsilon_T.
\]
Thus \(Q(d)-Q^\star\le\kappa+2\epsilon_T\). Finally,
\[
\widehat Q(d_T^\star)
\le
Q^\star+\epsilon_T
\le
Q(\widehat d_T)+\epsilon_T
\le
\widehat Q(\widehat d_T)+2\epsilon_T,
\]
so \(d_T^\star\in\widehat{\mathcal S}_{2\epsilon_T}\).
\end{proof}

\end{document}